\documentclass{article}

\usepackage{wrapfig}
\usepackage{float}
\usepackage[preprint]{neurips_2026}

\usepackage[utf8]{inputenc} 
\usepackage[T1]{fontenc}    
\usepackage{hyperref}       
\usepackage{url}            
\usepackage{booktabs}       
\usepackage{amsfonts}       
\usepackage{nicefrac}       
\usepackage{microtype}      
\usepackage{xcolor}         
\usepackage{enumitem}

\usepackage{amsmath}
\usepackage{subcaption}
\usepackage{algorithm}
\usepackage{algpseudocode}
\usepackage{xcolor}
\usepackage{booktabs}
\usepackage{colortbl}
\usepackage{multirow}
\usepackage{amsthm}
\usepackage{graphicx}
\usepackage{amssymb}  
\newtheorem{theorem}{Theorem}[section]
\newtheorem{lemma}[theorem]{Lemma}

\newtheorem{assumption}{Assumption}
\theoremstyle{definition}

\newtheorem{proposition}[theorem]{Proposition}
\theoremstyle{remark}
\newtheorem{remark}[theorem]{Remark}
\definecolor{param}{RGB}{244,204,204}
\definecolor{priv}{RGB}{252,229,205}
\definecolor{ours}{RGB}{217,234,211}

\title{Learned Neighbor Trust for Collaborative Deployment in Model-Agnostic Decentralized Learning}

\author{Michael Lanier, Luise Ge, Sastry Kompella,  Yevgeniy Vorobeychik \\
Department of Computer Science \\
Washington University in St. Louis and Nexcepta \\
lanier.m@wustl.edu, g.luise@wustl.edu, skompella@nexcepta.com, yvorobeychik@wustl.edu}

\begin{document}
\maketitle

\begin{abstract}
Many decentralized distillation methods are designed around training-time coordination,
yet deploy each node in isolation even when more capable neighbors remain available
at inference time. This is an incomplete objective for settings such as IoT, where
devices are heterogeneous, data is scarce and skewed, and a node's strongest
neighbors may far exceed its own local capacity. We study how nodes should train
so that their predictions compose well at deployment, and how each node should
learn whom to trust. Under a server-free, model-agnostic protocol where
nodes exchange only queries and soft predictions, we propose Learned Neighbor
Trust ($\bf{LNTrust}$) wherein each node learns a compact trust function over its neighborhood
from local validation evidence. This trust function gates auxiliary distillation
during training and defines a deployment ensemble at inference, so that
collaboration learned during training transfers directly to deployment. Across
datasets and topologies, LNTrust improves deployed accuracy over the strongest
output-only baseline by large margins while using significantly less communication
than previous methods.
\end{abstract}

\section{Introduction}
\label{sec:intro}

Decentralized learning is conventionally framed as a training-time
coordination
problem~\cite{mcmahan2017communication,lian2017can,kairouz2021advances,fallah2020personalized,li2021ditto,t2020personalized,marfoq2022knnper}:
nodes exchange parameters, gradients, or predictions to improve shared
or personalized models, on the assumption that a node's ceiling is set
by how well it converges. In many realistic deployments, such as edge~\cite{cisco_edge2024} and fog computing~\cite{cisco_fog2015}, this does not
hold: data are scarce, skewed, and many of the nodes (e.g., edge devices) can only deploy weaker architectures and can have limited communication.
On the other hand, pure inference-time ensembling of independently trained
models~\cite{seo2026federated} is bounded by limited local model capacity and communication constraints. 
Nonetheless, it is often the case that most devices have network neighbors that can deploy significantly larger local models, and can collect and store more data.
It is natural in such environments to develop schemes in which nodes can optimally leverage availability of high-quality neighbors, when available, for both training and inference.
We therefore ask: \emph{how should nodes
train so that their predictions compose well at deployment, and how
should each node learn to combine the predictions it receives?}

We study this under \emph{decentralized black-box peer querying}: nodes
query peers for predictions but never observe peer parameters,
gradients, or raw labels. 
This avoids strong assumptions such as identical model architectures for all nodes (an unreasonable requirement in decentralized settings), significantly reduces communication requirements, and does not require any knowledge about what models are being trained by individual nodes (a setting we refer to as \emph{model agnostic}).
As a consequence, exact joint optimization is not implementable,
because the cross-node supervised information it would require is never
available in one place. 
Instead, we make use of information 
accessible from local
evidence, learning deployment weights from local node-specific validation data, and
use training-time collaboration as an auxiliary signal anchored to each
node's supervised loss.

Specifically, we propose a \textbf{Learned Neighbor Trust} (LNTrust) approach. In LNTrust, each node fits a
compact trust model over its neighbors from relational statistics
(validation performance, prediction agreement, distributional overlap,
specialization) computed from query responses on local data. The
learned weights drive both a validated closed-neighborhood ensemble at
deployment and a gated auxiliary teacher during training, so that
collaboration helps without displacing local supervision.
A key advantage of our approach is its highly effective use of limited communication bandwidth, achieving an especially notable advantage over prior art in terms of efficacy per unit communication (Figure~\ref{fig:side_by_side}(a)).
In Figure~\ref{fig:side_by_side}(b), we visually illustrate that the approach successfully converges to a weight distribution that favors most useful neighbors, ignoring others either because their data is dissimilar from local, or because the models are too shallow.

\begin{figure}[t]
\centering
\begin{subfigure}[t]{0.48\textwidth}
  \centering
  \includegraphics[width=\linewidth]{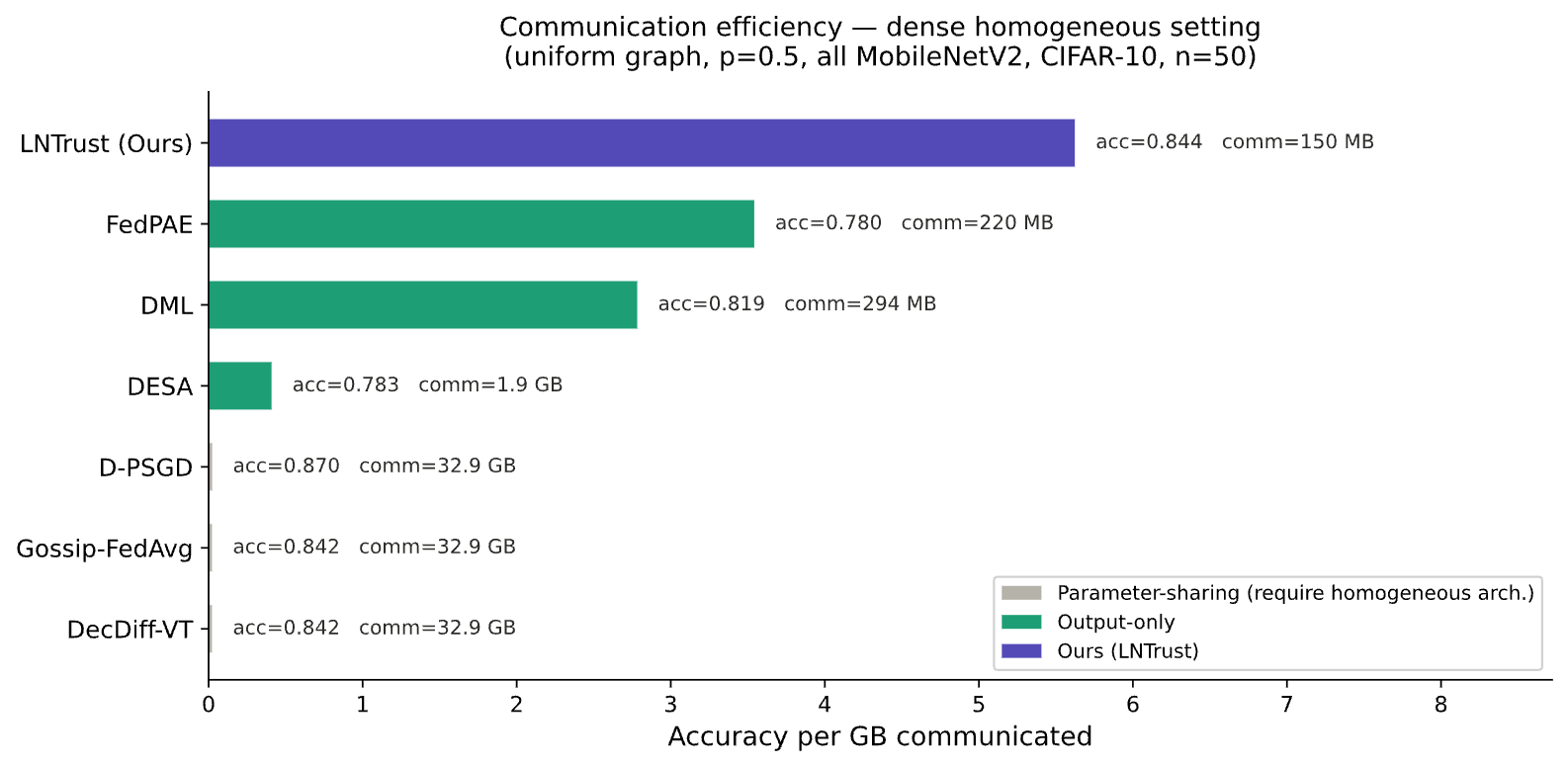}
  \caption{Communication efficiency, dense homogeneous setting.}
  \label{fig:comm_efficiency}
\end{subfigure}\hfill
\begin{subfigure}[t]{0.48\textwidth}
  \centering
  \includegraphics[width=\linewidth]{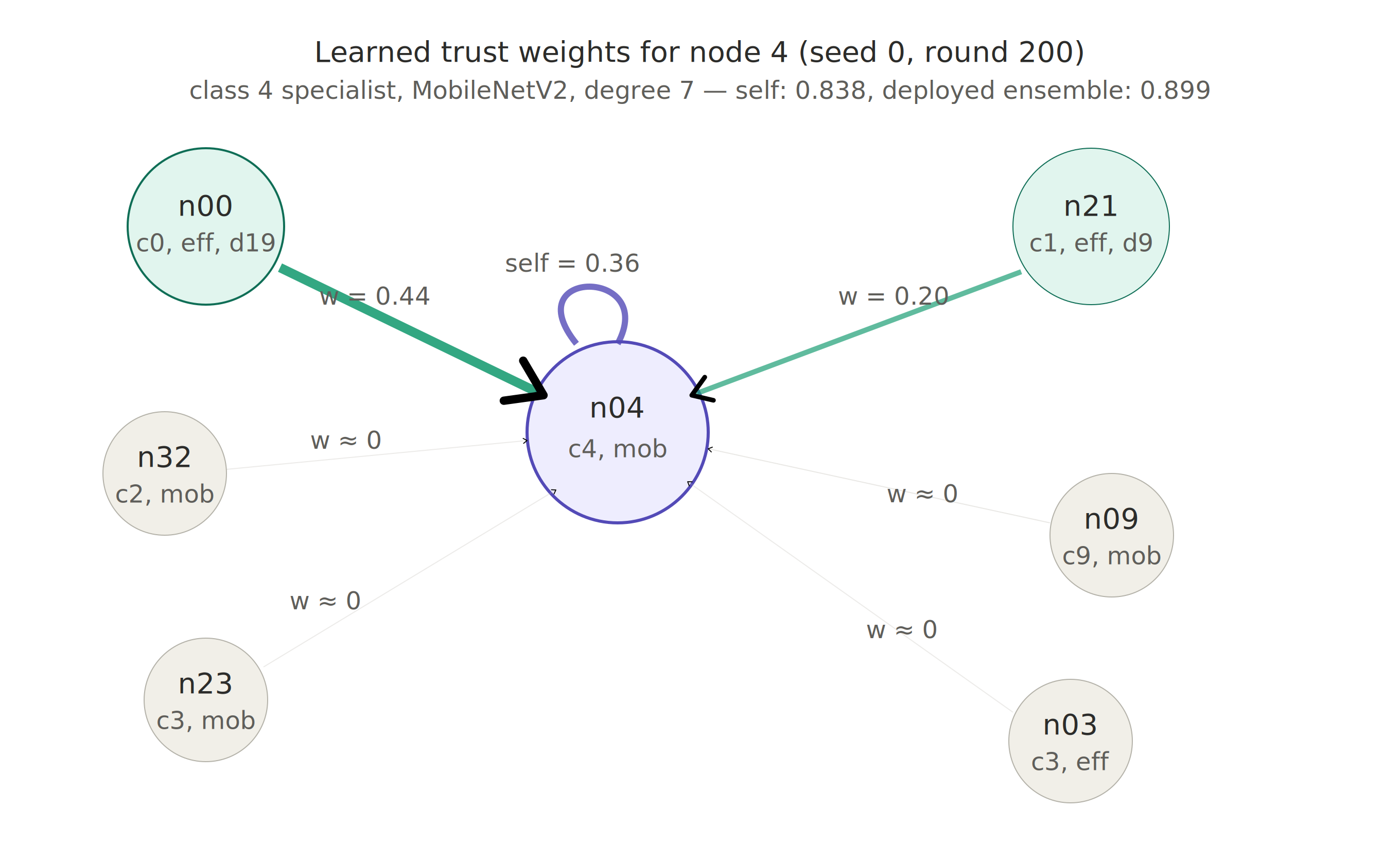}
  \caption{Deployment-time trust weights for node~4.}
  \label{fig:trust_example}
\end{subfigure}
\caption{%
  \textbf{Left:} Communication efficiency in the dense homogeneous
  setting (uniform graph, $p{=}0.5$, all MobileNetV2, CIFAR-10,
  $n{=}50$, 3 seeds). Bars show accuracy per GB communicated;
  annotations give absolute accuracy and per-node communication.
  Gray: parameter-sharing methods (require homogeneous
  architectures). Accuracy values match
  Table~\ref{tab:dense_homo}; this is a different regime from
  Table~\ref{tab:agnostic_skewed}.
  \textbf{Right:} Deployment-time trust weights for node~4
  (MobileNetV2, degree~7) in the sparse heterogeneous setting
  (seed~0, round~200); arrow thickness/color encode weight. The
  trust model concentrates weight on two EfficientNet hubs
  (nodes~0, 21) while suppressing four neighbors including another
  EfficientNet node, yielding $0.899$ test accuracy vs.\ $0.838$
  self-only.
}
\label{fig:side_by_side}
\end{figure}

\smallskip
\noindent\textbf{Summary of Contributions.}
\begin{itemize}[leftmargin=15pt,topsep=0pt,itemsep=0pt]
\item \textbf{Algorithm.} We develop LNTrust, a decentralized black-box peer
querying approach for node-specific collaboration. 
Each node decides
which neighbors to trust for its target distribution without
sharing parameters, gradients, or labels to drive both trust-gated auxiliary distillation during training and ensembling at deployment.
We demonstrate its effectiveness on
CIFAR-10, CIFAR-100, and EuroSAT.

\item \textbf{Theory.} We prove (i) a finite-sample deployment bound
showing validation-fit weights pay only
$\mathcal{O}\!\bigl(\varepsilon_{\mathrm{dep}}^{-1}\sqrt{\log|\bar{\mathcal{N}}(i)|/m_i}\bigr)$
excess log-loss over the best weight distribution, and (ii) a
controlled-perturbation bound showing LNTrust cannot drift far from
supervised training.

\item \textbf{Artifact.} We release a unified decentralized-learning
benchmarking harness that implements LNTrust alongside all evaluated
baselines (DML, DESA, FedPAE, Mean Teacher, D-PSGD, Gossip-FedAvg,
DecDiff-VT) under a common communication-accounted protocol, together
with the three graph configurations used in
this work.\footnote{https://github.com/Lan131/LnTrust}
\end{itemize}

\section{Related Work}

\noindent\textbf{Decentralized learning.}
Parameter and gradient sharing~\cite{mcmahan2017communication,li2020fedprox,
karimireddy2020scaffold,li2021moon} remains the dominant paradigm for
collaborative training without direct data pooling. FedAvg and successor
methods approximate joint optimization under a shared parameter space,
addressing non-IID instabilities through proximal regularization,
client-drift correction, and model-contrastive objectives. Fully decentralized
variants replace the central coordinator with gossip-style weight mixing:
D-PSGD~\cite{lian2017can} matches centralized convergence rates under
spectral conditions, and later work extends gossip to heterogeneous data
and topologies~\cite{hegedus2021gossip,yuan2023decdiff,ravikumar2024homogenizing}.
All of these methods exchange high-dimensional model state each round and
presuppose a common architecture across participants, a constraint LNTrust's
output-only protocol removes entirely.

\noindent\textbf{Output-only knowledge transfer.}
Output-only transfer replaces parameter exchange with soft-label communication and local distillation.
Server-mediated federated variants~\cite{li2019fedmd,lin2020ensemble,itahara2021distillation}
retain a central aggregator we do not assume. Fully decentralized instances
(DML~\cite{Zhang}, DeSA~\cite{huang2024overcoming}, IDKD~\cite{ravikumar2024homogenizing})
use peer predictions during training but deploy each model in isolation, so the
collaboration protocol and the deployment policy are designed independently. LNTrust
 reuses the same learned trust signal at both stages: the estimator that gates
distillation during training also drives combination at inference.

\noindent\textbf{Personalization in federated learning.}
A parallel line of work addresses the reality that a single global model rarely
suits all clients under heterogeneous data. Bi-level and proximal formulations bake
personalization into training through a regularizer that balances local fit against
a shared anchor, as in pFedMe~\cite{t2020personalized} and Ditto~\cite{li2021ditto}.
These methods still operate through parameter exchange, assume architectural
homogeneity, and incur additional optimization work per round. 

\noindent\textbf{Semi-supervised learning with teacher predictions.}
The use of teacher predictions on unlabeled data has a long history in
semi-supervised learning, most directly through consistency-target methods such as
Mean Teacher~\cite{TarvainenV17}. These approaches rely on a single internal teacher
(typically an EMA of the student itself) to generate soft targets on unlabeled
samples. LNTrust generalizes this idea to a decentralized multi-teacher setting
where candidate teachers are external peers with unknown data and unknown
architectures.

\noindent\textbf{Client-specific weighting at inference time.}
Inference-time ensembling of client predictions is itself not new. Server-based
methods aggregate client outputs
centrally~\cite{chen2021fedbe,lin2020ensemble,allouah2024fens,hamer2020fedboost};
personalization methods such as FedTHE~\cite{jiang2023fedthe}, ATP~\cite{bao2023atp},
and knn-Per~\cite{marfoq2022knnper} combine local and shared predictors under
various assumptions on parameter or representation sharing. Closest in spirit are
FedFomo~\cite{zhang2021fedfomo}, which uses local validation to weight peer models,
and FedPAE~\cite{mueller2024fedpae}, which selects peer ensembles under model
heterogeneity. Both download or inspect peer models; LNTrust combines peers through
black-box queries alone, and learns a trust function that drives both training-time
distillation and deployment-time combination from the same compact set of
relational features.

\section{Problem Setup}
\label{sec:setup}

Let $\mathcal{V}=\{1,\dots,n\}$ be the nodes of an undirected graph
$G=(\mathcal{V},\mathcal{E})$, each aiming to solve a $C$-class classification problem over the same input space $\mathcal{X} \subseteq \mathbb{R}^d$.
If $(i,j)\in\mathcal{E}$, nodes $i$ and $j$ can exchange messages.
We write $\mathcal{N}(i)$ for the open neighborhood of node $i$ and
$\bar{\mathcal{N}}(i)=\{i\}\cup\mathcal{N}(i)$ for its closed neighborhood.
Each node $i$ holds a labeled set
$\mathcal{D}_i^L=\{(x_k,y_k)\}$ drawn from a local distribution
$\mathcal{P}_i$ and an unlabeled shard $\mathcal{D}_i^U$ drawn from a
shared class-balanced pool.
Each node maintains a local predictor
$h_{\theta_i}:\mathbb{R}^d\rightarrow\mathbb{R}^C$ and associated class
probabilities $p_i(x)=\mathrm{softmax}(h_{\theta_i}(x)) \in\Delta^{C-1}$, with $p_j^y(x)$ the probability of $j$ predicting label $y$ for an input $x$.
We define $\theta = (\theta_1,\ldots,\theta_n)$.
Communication is \emph{output-only}: node $i$ sends unlabeled inputs to neighbor $j$, and $j$ returns soft predictions $p_j(x)$.
Ground-truth labels, model weights, gradients, and intermediate
activations are never transmitted.
This makes the protocol server-free and architecture-agnostic, but it
also means that node $i$ cannot directly optimize neighbor $j$'s model
using its private labels.
At deployment, node $i$ uses a node-specific ensemble over its closed neighborhood for inference. 

Our goal is to minimize the global
deployment objective
\begin{equation}
\mathcal F(\theta,\alpha)
\equiv
\sum_{i\in\mathcal V}
\mathbb E_{(x,y)\sim\mathcal{P}_i}
\bigl[\ell(q_i(y\mid x;\alpha_i,\theta))\bigr],\ \
\mathrm{where}\ \
q_i(y\mid x;\alpha_i,\theta)
=
\sum_{j\in\bar{\mathcal N}(i)} \alpha_{ij}\, p_j^y(x),
\label{eq:deployment-objective}
\end{equation}
subject to $\alpha_i \in \Delta(\bar{\mathcal N}(i))$ for every node $i$, for some loss function $\ell(\cdot)$.

\section{Solution Approach}
\label{sec:method}

If we had centralized access to all node-level models, we could optimize $\mathcal F(\theta,\alpha)$ using alternating minimization: alternating between minimizing weighted average loss given fixed weights, and optimizing the weights in isolation.
However, our goal is to solve this problem in an entirely decentralized fashion, with individual nodes  optimizing their respective $h_{\theta_i}$ locally.
In addition, when nodes have particularly small local datasets, the free simplex fit is high
variance and often over-specializes to a few held-out examples.
Consequently, an alternative approach is warranted.
To this end, we propose a two-stage training protocol described next.

\subsection{Two-Stage Training Protocol}
\label{sec:twostage}

Training proceeds in two stages.
In Stage~1, each node trains locally without communication using the
supervised objective
\begin{equation*}
\mathcal{L}_i^{\mathrm{sup}}(\theta_i)
\equiv
\sum_{(x,y)\in\mathcal{D}_i^L}
\!\left[\ell\!\left(h_{\theta_i}(x),y\right)\right]
+ \frac{\lambda}{2}\|\theta_i\|^2.
\end{equation*}
After $T_1$ rounds, node $i$ constructs a held-out validation set
$\mathcal{D}_i^V$\footnote{This must be carefully done to ensure that $\mathcal{D}_i^V$ is really held out. See section 5 for details.} matched to its local class distribution. This set is
used only to estimate neighbor quality and to fit the trust model; the
corresponding examples are excluded from Stage~2 supervised updates, so
trust estimation has a real labeled-data cost.

Stage 2 runs for $T_2$ rounds. Each round consists of (i) supervised
updates on the remaining labeled data, (ii) trust-gated
distillation from neighbor predictions on unlabeled data, and (iii)
periodic re-probing and trust-model retraining every $F_{\mathrm{trust}}$
rounds. Distillation here means training node $i$'s model on soft
predictions returned by neighbors on unlabeled examples, weighted by the
learned trust scores defined in Section~\ref{sec:trust}. This separation
keeps local supervision primary while allowing nodes to borrow signal only
when neighbors validate well on their target classes. This distillation step
uses an unlabeled pool, which is common in this setting (although our approach \emph{does not require this} and remains effective, as shown in Table~\ref{tab:ablation_pseudo}).
We describe all steps of stage 2 in detail below.

\subsection{Probe Responses and Learned Trust}
\label{sec:trust}

After Stage~1, node $i$ queries every
$j\in\bar{\mathcal{N}}(i)$ on the images in $\mathcal{D}_i^V$.
Using its private labels $y_i^V$, node $i$ computes the per-class probe
response
\begin{equation*}
\rho_{ij}^c
=
\frac{
\sum_{(x,y)\in\mathcal{D}_i^V}
\mathbf{1}[y=c]\,\mathbf{1}[\hat y_j(x)=c]
}{
\sum_{(x,y)\in\mathcal{D}_i^V}\mathbf{1}[y=c]
},
\qquad
\hat y_j(x)=\operatorname*{argmax}_{c'} p_j^{c'}(x).
\end{equation*}
The vector $\rho_{ij}\in[0,1]^C$ measures how useful neighbor $j$ is on
the classes that matter to node $i$.

Node $i$ also estimates neighbor $j$'s class distribution from
predictions on a random unlabeled subset
$\mathcal{S}\subseteq\mathcal{D}_i^U$:
\begin{equation*}
\hat w_j^c
=
\frac{1}{|\mathcal{S}|}
\sum_{x\in\mathcal{S}}
\mathbf{1}\!\left[\operatorname*{argmax}_{c'} p_j^{c'}(x)=c\right].
\end{equation*}
Let $w_i$ denote node $i$'s empirical training class distribution.
From $w_i$, $\hat w_j$, and the probe statistics, we build the
six-dimensional relational feature vector including the normalized degree of neighbor $j$, $d_j/n$: 
\begin{equation*}
\phi_{ij}
=
\bigl[
\textstyle\sum_c \min(w_i^c,\hat w_j^c),\,
\bar\rho_{ij},\,
\tilde\rho_{ij},\,
\mathrm{KL}(w_i\|\hat w_j),\,
H(\hat w_j),\,
d_j/n
\bigr]
\end{equation*}
where $\bar\rho_{ij}=C^{-1}\sum_c \rho_{ij}^c$ and
$\tilde\rho_{ij}=\sum_c w_i^c\rho_{ij}^c$.
These features summarize specialization, distributional alignment, and
probe performance without introducing a class-dependent input dimension.

Each node fits a local trust model
$g_{\psi_i}:\mathbb{R}^6\rightarrow\mathbb{R}$. 
Given a set of options (neighbors) $\mathcal{A}_i$, trust weights are
\begin{equation}
\alpha_{ij}
=
\frac{\exp(g_{\psi_i}(\phi_{ij}))}
{\sum_{k\in\mathcal{A}_i}\exp(g_{\psi_i}(\phi_{ik}))}
\label{eq:trust}
\end{equation}
We use $\mathcal{A}_i=\mathcal{N}(i)$ for distillation and
$\mathcal{A}_i=\bar{\mathcal{N}}(i)$ for deployment.
Node $i$ trains $\psi_i$ on its own validation labels by minimizing the
cross-entropy of the trust-weighted ensemble on $\mathcal{D}_i^V$:
\begin{equation*}
\min_{\psi_i}\;
\ell\!\left(
\sum_{j\in\bar{\mathcal{N}}(i)}
\alpha_{ij}\,p_j(\mathcal{D}_i^V),\,
y_i^V
\right)
\end{equation*}

\subsection{Trust-Gated Distillation and Deployment}
\label{sec:distill}

A significant consideration in many decentralized learning settings, such as IoT, is limited communication bandwidth.
We impose an exogeneous communication cost budget as follows.
In Stage~2, node $i$ queries neighbors on an unlabeled batch
$\mathcal{B}_i^U\subseteq\mathcal{D}_i^U$ and forms the trust-weighted
ensemble
\begin{equation}
\bar p_i(x)
=
\sum_{j\in\mathcal{N}(i)} \alpha_{ij}\,p_j(x)
\label{eq:ensemble}
\end{equation}
The size of this batch is a training-time hyperparameter we call the
\emph{per-round pseudo-label budget}, $B=|\mathcal{B}_i^U|$, held fixed
across nodes and rounds. 
$B$ controls how much unlabeled data each node
consumes per distillation step, imposing a communication budget, and therefore how much neighbor signal
can be absorbed before the confidence filter and gate below take effect.
Setting $B=0$ disables distillation entirely, in which case
equation~\eqref{eq:total} reduces to the supervised loss and LNTrust
degenerates to local training augmented by the deployment-time trust
ensemble. This boundary case is the $0$-pseudo row in
Table~\ref{tab:ablation_pseudo} and isolates the deployment contribution
from distillation.

\paragraph{Negative-transfer gate.}
To suppress negative transfer, we compare the validation performance of
the ensemble against node $i$'s own model using distribution-weighted
probe accuracy:
\begin{equation*}
a_i^{\mathrm{self}}
=
\sum_c w_i^c \rho_{ii}^c,
\qquad
a_i^{\mathrm{ens}}
=
\sum_{j\in\mathcal{N}(i)}
\alpha_{ij}\sum_c w_i^c \rho_{ij}^c.
\end{equation*}
The effective distillation weight is then
\begin{equation*}
\lambda_i^{\mathrm{eff}}
=
\lambda_{\mathrm{distil}}
\cdot
\min\!\left(
1,\frac{a_i^{\mathrm{ens}}}{a_i^{\mathrm{self}}+\epsilon}
\right)
\end{equation*}
where $\lambda_{\mathrm{distil}}$ is a hyperparmeter. If the ensemble is weaker than $i$'s own model on its target classes,
$\lambda_i^{\mathrm{eff}}$ shrinks proportionally; if it is stronger,
the full distillation weight $\lambda_{\mathrm{distil}}$ is applied.

\paragraph{Confidence filter.}
We retain only examples on which the ensemble is confident:
\begin{equation*}
\mathcal{B}_i^*
=
\bigl\{
x\in\mathcal{B}_i^U:
\max_c \bar p_i^c(x)>\tau(C)
\bigr\},
\qquad
\tau(C)
=
\max\!\bigl(\tau_{\mathrm{abs}},\,\tfrac{1}{C}+\tau_{\mathrm{conf}}\bigr),
\end{equation*}
combining an absolute floor $\tau_{\mathrm{abs}}$ with a margin
$\tau_{\mathrm{conf}}$ above the uniform level $1/C$. The floor prevents the filter from collapsing to a trivial constraint
as $C$ grows.

\paragraph{Distillation loss.}
We consider two instantiations of the distillation loss and treat the
choice as a hyperparameter of the training run, selected \emph{per
experiment based on the number of classes}. In both variants we use an
importance weight
\begin{equation}
\omega_i(x)=C\,w_i^{\hat c(x)},
\qquad
\hat c(x)=\operatorname*{argmax}_c \bar p_i^c(x),
\label{eq:iw}
\end{equation}
which upweights examples whose predicted class matches $i$'s target
distribution.

\smallskip
\noindent\textit{Hard pseudo-labels (used by default for small $C$).}
The student is trained against the ensemble argmax,
\begin{equation*}
\mathcal{L}_i^{\mathrm{distil-hard}}(\theta_i)
=
\frac{1}{|\mathcal{B}_i^*|}
\sum_{x\in\mathcal{B}_i^*}
\omega_i(x)\;
\ell\!\left(
h_{\theta_i}(x),\,
\hat c(x)
\right)
\end{equation*}
This is our default at $C{=}10$: hard labels are easier to learn
from per example and combine cleanly with the per-class importance
weight.

\smallskip
\noindent\textit{Soft distillation (used on occasion, typically for
larger $C$).}
The student is trained with KL divergence against the full ensemble
soft target,
\begin{equation*}
\mathcal{L}_i^{\mathrm{distil-soft}}(\theta_i)
=
\frac{\alpha_{\mathrm{KL}}}{|\mathcal{B}_i^*|}
\sum_{x\in\mathcal{B}_i^*}
\omega_i(x)\;
\mathrm{KL}\!\left(
\bar p_i(x)\,\big\|\,
\mathrm{softmax}(h_{\theta_i}(x))
\right),
\end{equation*}
with an additional scaling $\alpha_{\mathrm{KL}}\in(0,1]$ reflecting the
fact that KL gradients on a $C$-dimensional simplex carry more signal
per example than a single hard label. We select the soft variant when
the argmax is an unreliable summary of the ensemble, for instance when
$C$ is large enough that many classes carry non-trivial probability
mass, and use $\mathcal{L}_i^{\mathrm{distil}}
\in\{\mathcal{L}_i^{\mathrm{distil-hard}},
\mathcal{L}_i^{\mathrm{distil-soft}}\}$
as a single symbol in what follows.

\paragraph{Per-round objective.}
The per-round objective is
\begin{equation}
\mathcal{L}_i(t)
=
\mathcal{L}_i^{\mathrm{sup}}(\theta_i)
+
\lambda_i^{\mathrm{eff}}\,
\mathcal{L}_i^{\mathrm{distil}}(\theta_i),
\label{eq:total}
\end{equation}
implemented as sequential supervised and distillation updates.

\paragraph{Deployment.}
At test time, node $i$ forms a closed-neighborhood ensemble with
deployment weights $\alpha_{ij}^{\mathrm{deploy}}$ computed from
equation~\eqref{eq:trust} using $\mathcal{A}_i=\bar{\mathcal{N}}(i)$:
\begin{equation}
p_i^{\mathrm{deploy}}(x)
=
\sum_{j\in\bar{\mathcal{N}}(i)}
\alpha_{ij}^{\mathrm{deploy}}\,p_j(x)
\label{eq:deploy}
\end{equation}
We then evaluate the analogous distribution-weighted validation score:
\begin{equation*}
a_i^{\mathrm{deploy}}
=
\sum_{j\in\bar{\mathcal{N}}(i)}
\alpha_{ij}^{\mathrm{deploy}}
\sum_c w_i^c \rho_{ij}^c
\end{equation*}

\section{Analysis}
\label{sec:theory}

We provide two complementary theoretical guarantees: (i) a unified oracle inequality for the deployment risk of validation-fit trust weights, and (ii) a controlled-perturbation bound on the Stage 2 training trajectory. Together, these results demonstrate that LNTrust deployment is statistically well-posed and that the collaborative training phase remains stable even under weak peer signals.

\paragraph{Deployment risk.}
Let $\theta = (\theta_1,\dots,\theta_n)$ be the predictors at deployment. For node $i$ and trust weights $\alpha_i \in \Delta(\overline{\mathcal{N}}(i))$, the collaborative ensemble prediction is $q_i(y\mid x;\theta,\alpha_i) = \sum_{j\in\overline{\mathcal{N}}(i)} \alpha_{ij}\,p^y_j(x;\theta_j)$. We evaluate this ensemble using the clipped log-loss $\ell_{\epsilon_{\mathrm{dep}}}(u) = -\log\max\{u,\epsilon_{\mathrm{dep}}\}$ for fixed $\epsilon_{\mathrm{dep}}\in(0,1]$. The true deployment risk is:
\[
R^{\mathrm{dep}}_{i,\epsilon}(\theta,\alpha_i) 
\;:=\; 
\mathbb{E}_{(X,Y)\sim P_i}\!\left[\ell_{\epsilon_{\mathrm{dep}}}\!\big(q_i(Y\mid X;\theta,\alpha_i)\big)\right],
\]
with $\widehat R^{V,\mathrm{dep}}_{i,\epsilon}$ denoting the empirical risk on local validation set $\mathcal{D}^V_i$ ($|\mathcal{D}^V_i|=m_i$). 

LNTrust parameterizes these weights via the class $\mathcal{A}_i^{\mathrm{LN}} := \{ \alpha_i(g) : g \in \mathcal{G} \}$, where $g$ is a trust-MLP from class $\mathcal{G}$. Let $\hat\alpha^{\mathrm{dep}}_i$ be a $\xi_i$-suboptimal empirical minimizer of $\widehat R^{V,\mathrm{dep}}_{i,\epsilon}(\theta,\cdot)$ over $\mathcal{A}_i^{\mathrm{LN}}$.

\begin{assumption}[Finite-neighborhood Interpolation]
\label{ass:expressivity}
The relational features $\{\phi_{ij}: j\in\overline{\mathcal{N}}(i)\}$ are pairwise distinct, and $\mathcal{G}$ is sufficiently expressive to interpolate arbitrary real values on this finite set.
\end{assumption}

Assumption~\ref{ass:expressivity} is benign in decentralized settings: the features include node-specific statistics (e.g., KL divergence to local class distributions) that make exact collisions a measure-zero event. Furthermore, standard ReLU MLPs satisfy the interpolation property for any neighborhood of size $d_i$ with a single hidden layer of size $d_i$.

\begin{proposition}[Deployment Bound]
\label{thm:deploy_oracle}
Under Assumption~\ref{ass:expressivity}, for any $\delta\in(0,1)$, the following holds with probability at least $1-\delta$ over the draw of $\mathcal{D}^V_i$:
\[
R^{\mathrm{dep}}_{i,\epsilon}(\theta, \hat\alpha^{\mathrm{dep}}_i)
\;\le\;
\inf_{\alpha_i \in \Delta(\overline{\mathcal{N}}(i))} R^{\mathrm{dep}}_{i,\epsilon}(\theta,\alpha_i)
\;+\; \xi_i
\;+\; \frac{4}{\epsilon_{\mathrm{dep}}}\sqrt{\frac{2\log|\overline{\mathcal{N}}(i)|}{m_i}}
\;+\; \sqrt{\frac{2\log^2(1/\epsilon_{\mathrm{dep}})\log(2/\delta)}{m_i}}.
\]
\end{proposition}

The proof (Appendix~\ref{app:proof_deploy}) leverages the fact that under Assumption~\ref{ass:expressivity}, the MLP parameterization is $\ell_1$-dense in the simplex $\Delta(\overline{\mathcal{N}}(i))$ (Proposition~\ref{prop:trust-expressivity}), allowing us to bound the risk relative to the \textit{absolute best} simplex weights rather than just the best-in-class weights. Because the ensemble is constrained to the simplex, the Rademacher complexity depends only on the number of neighbors, ensuring that using a high-capacity trust MLP does not lead to overfitting on the validation set.

\paragraph{Training as a controlled perturbation.}
Since black-box queries preclude direct optimization of the global objective, Stage~2 uses peer predictions only as a gated auxiliary signal. We quantify the stability of this process by comparing the LNTrust trajectory $\{\theta_{i,t}\}$ to a "self-only" trajectory $\{\vartheta_{i,t}\}$—initialized identically but trained without distillation. Under the smoothness assumptions in Appendix~\ref{app:perturbation}, we have:
\[
\|\theta_{i,T} - \vartheta_{i,T}\|
\;\le\;
\exp\!\Big(\textstyle\sum_{s<T} \eta^{\sup}_s \beta_{i,s}\Big)
\textstyle\sum_{t<T} \eta^{\mathrm{dist}}_t \, \lambda^{\mathrm{eff}}_{i,t} \, G_{i,t},
\]
where $\beta_{i,t}$ is the supervised smoothness constant, $G_{i,t}$ is the distillation gradient norm, and $\lambda^{\mathrm{eff}}_{i,t}$ is the effective trust gate. 

Crucially, the gate $\lambda^{\mathrm{eff}}_{i,t}$ is selective: its population version is monotone non-increasing in the peer ensemble's excess error (Lemma~\ref{lem:gate_selectivity}). When neighbors provide poor predictive utility, the gate shrinks and the collaborative perturbation collapses toward the local supervised solution. This ensures that collaboration is "safe"—even in highly heterogeneous neighborhoods, a node's local performance is protected from catastrophic divergence.
\section{Experiments}
\label{sec:experiments}
\smallskip
\noindent\textbf{Datasets and topologies.}
We evaluate on CIFAR-10, CIFAR-100, and EuroSAT with $n{=}50$ nodes under
three configurations. The \emph{sparse heterogeneous} setting (CIFAR)
combines round-robin label skew (skew 10; Dirichlet size
$\alpha_{\mathrm{size}}{=}0.05$) with a Barabasi--Albert graph
($m{=}2$, avg.\ degree ${\approx}4$), and a 90/10 architecture mix
(MobileNetV2 / EfficientNet-B0 at hubs) that stresses the
architecture-agnostic protocol. The \emph{dense homogeneous} setting
(CIFAR) uses a uniform random graph ($p{=}0.5$, avg.\ degree ${\approx}25$)
with all-MobileNetV2 nodes, which is favorable to parameter-sharing
baselines. The \emph{geographic} setting (EuroSAT) derives both data and
graph from the Sentinel-2 layout: each node owns a region ($k$-means on
tile coordinates, $k{=}50$) and edges follow distance-decayed preferential
attachment (97 edges, degree range 2--13); hubs again run EfficientNet-B0
(Figure~\ref{fig:eurosat_geo}). Each node also holds 2{,}000
class-balanced unlabeled examples for distillation and probing. Full
dataset, topology, and hyperparameter details are in
Appendix~\ref{app:hyperparams}.

\begin{figure}[htbp]
\centering
\begin{minipage}[b]{0.46\linewidth}
  \centering
  \includegraphics[height=4.5cm]{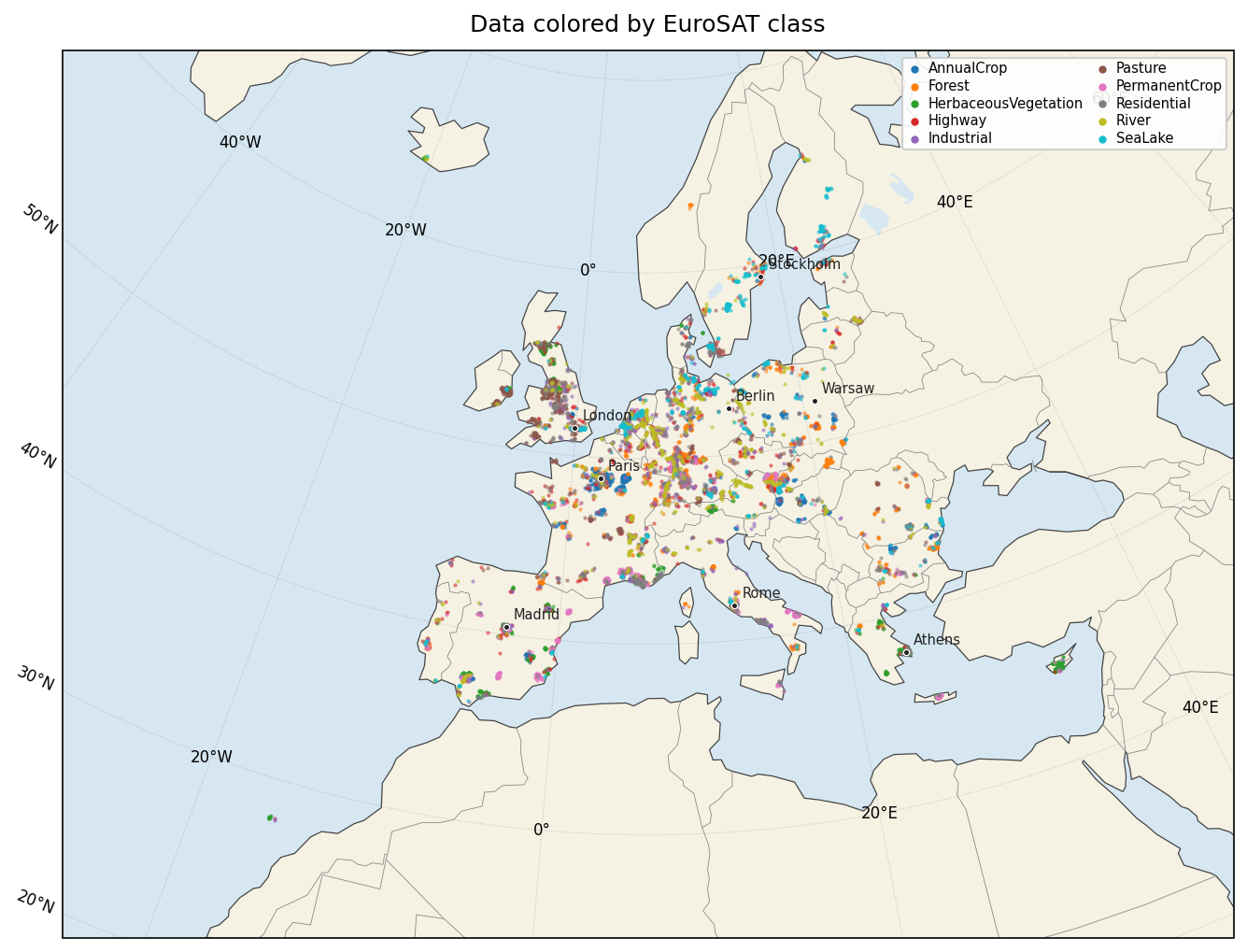}
\end{minipage}
\hfill
\begin{minipage}[b]{0.53\linewidth}
  \centering
  \includegraphics[height=4.5cm]{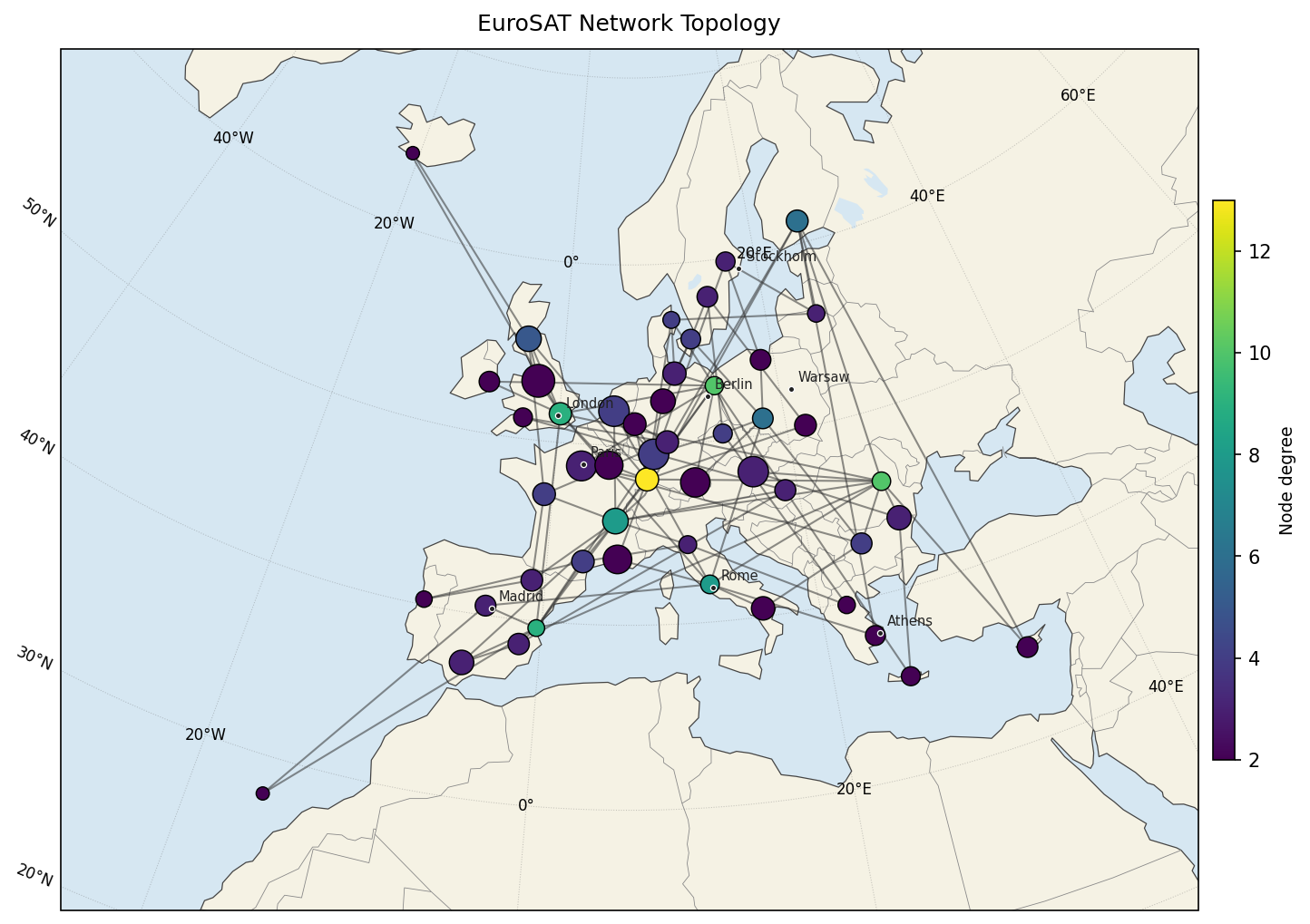}
\end{minipage}
\caption{EuroSAT geographic setup. \textbf{Left:} per-tile land-cover
class across Europe; each point is one Sentinel-2 tile colored by its
EuroSAT class label, and regional clustering of classes produces the
per-node label skew that drives heterogeneity in the geographic
setting. \textbf{Right:} 50-region communication graph obtained by
$k$-means clustering of Sentinel-2 tile coordinates, connected by
distance-decayed preferential attachment ($m{=}2$, 97 edges); node
size $\propto$ tiles per region and color encodes degree. Per-region
class distributions (Appendix~\ref{app:topology}) drive the geographic
heterogeneity used in Table~\ref{tab:eurosat_geo}.}
\label{fig:eurosat_geo}
\end{figure}

\smallskip
\noindent\textbf{Baselines.}
We compare against output-only and parameter-sharing methods.
\textbf{DML}~\cite{Zhang} uses mutual KL divergence on shared predictions;
\textbf{DESA}~\cite{huang2024overcoming} adds peer-to-peer anchor synthesis
and feature alignment; \textbf{FedPAE}~\cite{mueller2024fedpae} downloads
peer models and constructs a validation-optimized deployment ensemble without
collaborative training; \textbf{Mean Teacher}~\cite{TarvainenV17} ($p{=}0$)
uses an EMA consistency target.
\textbf{Independent Learning} ($p{=}0$) is the lower bound.
The dense homogeneous setting additionally includes parameter-sharing baselines
requiring architectural homogeneity: \textbf{D-PSGD}~\cite{lian2017can},
\textbf{Gossip-FedAvg}~\cite{hegedus2021gossip}, and \textbf{DecDiff-VT}~\cite{yuan2023decdiff}.
We exclude server-mediated methods (FedMD~\cite{li2019fedmd}, FedDF~\cite{lin2020ensemble})
and IDKD~\cite{ravikumar2024homogenizing}, which requires a public \emph{labeled} dataset.

\noindent\textbf{Communication model.}
Each node caches backbone features locally; when node~$i$ queries neighbor~$j$,
$j$ returns $C$-dimensional logit vectors. Both LNTrust and DML exchange only
logits. DESA additionally transmits synthesized anchor images and $d$-dimensional
feature vectors, incurring higher communication (Figure~\ref{fig:comm_efficiency}).

\noindent\textbf{Validation protocol and trust-estimation cost.}
All methods use the same Stage~1: each node trains on its full labeled
allocation with no data withheld.
After Stage~1, each node draws a fresh validation set
$\mathcal{D}_i^V$ from the unused global pool, disjoint from all training,
test, and unlabeled sets, and matched to its class distribution
$\mathcal{P}_i$.
These held-out examples are used only for probe responses and trust-model
training, avoiding leakage.
To prevent indirect leakage in Stage~2, the corresponding training examples
are removed, reducing each node's Stage~2 labeled budget.
Thus LNTrust pays a real cost: it trades supervised Stage~2 signal for trust
estimation.
This is reflected in the \(B=0\) row of Table~6, where the self-only model is
weaker than Independent Learning, while the validation-calibrated deployment
ensemble more than recovers the loss.

\noindent\textbf{Metrics.}
We report two metrics averaged over nodes and three random seeds.
\textbf{Test} is the primary metric: the accuracy of the trust-gated
ensemble each node deploys (for baselines, Test\,=\,Self since each
baseline node deploys only its own local model).
\textbf{Self} reports accuracy when every node deploys its own local model,
isolating self deployment vs the deployment described above.

\subsection{Results}
%
%

\begin{table}[t]
\centering
\caption{Sparse heterogeneous setting on CIFAR-10 and CIFAR-100:
BA graph ($m{=}2$, $p{=}0.1$), 90\% MobileNetV2 + 10\% EfficientNet-B0
at hubs, $n=50$ nodes. CIFAR-10 uses skew factor 10; CIFAR-100 uses
skew factor 5. Parameter-sharing methods cannot operate due to
architectural heterogeneity. Mean $\pm$ std over 3 seeds.}
\label{tab:agnostic_skewed}
\setlength{\tabcolsep}{4pt}
\begin{tabular}{lcccc}
\toprule
& \multicolumn{2}{c}{\textbf{CIFAR-10}}
& \multicolumn{2}{c}{\textbf{CIFAR-100}} \\
\cmidrule(lr){2-3}\cmidrule(lr){4-5}
\textbf{Method} & \textbf{Test} & \textbf{Self}
                & \textbf{Test} & \textbf{Self} \\
\midrule
Indep.\ Learning ($p{=}0$)
  & \multicolumn{2}{c}{$0.791 \pm 0.022$}
  & \multicolumn{2}{c}{$0.550 \pm 0.025$} \\
Mean Teacher~\cite{TarvainenV17} ($p{=}0$)
  & \multicolumn{2}{c}{$0.782 \pm 0.016$}
  & \multicolumn{2}{c}{$0.553 \pm 0.030$} \\
\midrule
DML~\cite{Zhang}
  & $0.795 \pm 0.015$ & $0.795 \pm 0.015$
  & $0.569 \pm 0.012$ & $0.569 \pm 0.012$ \\
DESA~\cite{huang2024overcoming}
  & $0.790 \pm 0.017$ & $0.790 \pm 0.017$
  & $0.565 \pm 0.020$ & $0.565 \pm 0.020$ \\
FedPAE~\cite{mueller2024fedpae}
  & $0.739 \pm 0.008$ & $0.739 \pm 0.008$
  & $0.549 \pm 0.005$ & $0.549 \pm 0.005$ \\
\midrule
\rowcolor{ours}
\textbf{Ours (LNTrust)}
  & $\mathbf{0.865 \pm 0.017}$ & $\mathbf{0.829 \pm 0.006}$
  & $\mathbf{0.615 \pm 0.015}$ & $0.557 \pm 0.009$ \\
\bottomrule
\end{tabular}
\end{table}

LNTrust improves on every output-only baseline across all three
settings while remaining competitive with parameter-sharing methods
where those can operate. In the sparse heterogeneous setting
(Table~\ref{tab:agnostic_skewed}; BA graph, 90/10
MobileNetV2/EfficientNet-B0), where parameter-sharing is infeasible,
LNTrust reaches $0.865$ on CIFAR-10 and $0.615$ on CIFAR-100 versus a
best baseline of $0.795$ and $0.569$, respectively, an improvement of over 8\% in both cases.

\enlargethispage{4\baselineskip}
\begin{wraptable}[13]{r}{0.56\textwidth}
\centering
\vspace{-\baselineskip}
\caption{EuroSAT geographic: 50 regions,
distance-decayed BA ($m{=}2$, 97 edges).}
\label{tab:eurosat_geo}
\setlength{\tabcolsep}{3pt}
\begin{tabular}{lcc}
\toprule
\textbf{Method} & \textbf{Test} & \textbf{Self} \\
\midrule
Indep.\ ($p{=}0$) & \multicolumn{2}{c}{$0.942 \pm 0.011$} \\
Mean Teacher~\cite{TarvainenV17} & \multicolumn{2}{c}{$0.947 \pm 0.005$} \\
\midrule
DML~\cite{Zhang}                 & $0.947 \pm 0.004$ & $0.947 \pm 0.004$ \\
DESA~\cite{huang2024overcoming}  & $0.946 \pm 0.004$ & $0.946 \pm 0.004$ \\
FedPAE~\cite{mueller2024fedpae}  & $0.948 \pm 0.004$ & $0.947 \pm 0.004$ \\
\midrule
\textbf{Ours} & $\mathbf{0.967 \pm 0.013}$ & $\mathbf{0.967 \pm 0.013}$ \\
\bottomrule
\end{tabular}
\end{wraptable}
In the geographic EuroSAT setting (Table~\ref{tab:eurosat_geo}), all output-only baselines
cluster within ${\approx}0.5$\,pp of Independent Learning's $94.2\%$, while LNTrust reaches
$96.7\%$. The near-identical Self and Test scores indicate the gain comes from Stage~2
distillation: the trust mechanism correctly recovers high self-weight when neighbors offer
little advantage.

The dense homogeneous setting (Table~\ref{tab:dense_homo}; uniform graph, all MobileNetV2)
is the regime most favorable to parameter-sharing: dense connectivity ($p{=}0.5$) and
homogeneous architecture make weight-space averaging a near-optimal variance reducer across
redundant local optima. Even here, LNTrust ($0.844$) outperforms every output-only baseline
and matches Gossip-FedAvg and DecDiff-VT while transmitting no weights and supporting
heterogeneous architectures by construction. D-PSGD ($0.870$) leads; weight-space mixing
plausibly helps more under architectural homogeneity, and we read this gap as the price
LNTrust pays for the restrictions it removes elsewhere. D-PSGD cannot participate in the
sparse heterogeneous or geographic settings due to mixed-architecture nodes, so the
comparison is genuinely setting-specific rather than a global ranking, and the relative
position of parameter-sharing and output-only methods inverts as soon as the homogeneity
assumption is dropped.

\begin{table}[t]
\centering
\caption{Dense homogeneous: uniform graph ($p{=}0.5$), all MobileNetV2.
CIFAR-10, $n{=}50$, 3 seeds. Italicized methods require architectural
homogeneity and transmit full weights.}
\label{tab:dense_homo}
\setlength{\tabcolsep}{4pt}
\begin{tabular}{lcc}
\toprule
\textbf{Method} & \textbf{Test} & \textbf{Self} \\
\midrule
Indep.\ Learning ($p{=}0$) & \multicolumn{2}{c}{$0.775 \pm 0.020$} \\
Mean Teacher~\cite{TarvainenV17} ($p{=}0$) & \multicolumn{2}{c}{$0.784 \pm 0.029$} \\
\midrule
\multicolumn{3}{l}{\textit{Parameter-sharing}} \\
\textit{Gossip-FedAvg}~\cite{hegedus2021gossip} & \textit{$0.842 \pm 0.019$} & \textit{$0.842 \pm 0.019$} \\
\textit{DecDiff-VT}~\cite{yuan2023decdiff}      & \textit{$0.842 \pm 0.022$} & \textit{$0.842 \pm 0.022$} \\
\textit{D-PSGD}~\cite{lian2017can}              & \textit{$\mathbf{0.870 \pm 0.015}$} & \textit{$\mathbf{0.870 \pm 0.015}$} \\
\midrule
\multicolumn{3}{l}{\textit{Output-only}} \\
DML~\cite{Zhang}                 & $0.819 \pm 0.018$ & $0.819 \pm 0.018$ \\
DESA~\cite{huang2024overcoming}  & $0.783 \pm 0.008$ & $0.783 \pm 0.008$ \\
FedPAE~\cite{mueller2024fedpae}  & $0.780 \pm 0.004$ & $0.780 \pm 0.004$ \\
\midrule
\rowcolor{ours}
\textbf{Ours (LNTrust)} & $\mathbf{0.844 \pm 0.019}$ & $\mathbf{0.803 \pm 0.014}$ \\
\bottomrule
\end{tabular}
\end{table}

\noindent\textbf{Communication efficiency.}
Figure~\ref{fig:comm_efficiency} shows accuracy per GB in the dense homogeneous setting.
LNTrust exchanges only unidirectional $C$-dimensional logits, versus bidirectional logits
(DML), anchor images and feature vectors (DESA), or full weights per round (parameter-sharing
methods). LNTrust achieves the highest efficiency ratio among all collaborative methods.

\noindent\textbf{Trust mechanism ablation.}
Table~\ref{tab:ablation_trust} replaces the trust MLP with three scalar-signal alternatives
(LinUCB, Weighted Majority, Hedge) and a free simplex ERM fit, holding the rest of the pipeline
fixed. LNTrust outperforms all four, confirming that structured relational features generalize
better than scalar reward accumulation or unregularized simplex weights under limited validation.
Additional ablations are in the appendix.

\begin{table}[t]
\centering
\caption{Neighbor weighting ablation (sparse heterogeneous setting).
All variants use the same output-only communication protocol, trust
gate, distillation loss, and BA graph ($m{=}2$, $p{=}0.1$) with
heterogeneous architectures. \textbf{Signal} indicates what
information drives the weight update. CIFAR-10, $n=50$ nodes, mean
$\pm$ std over 3 seeds.}
\label{tab:ablation_trust}
\setlength{\tabcolsep}{3pt}
\begin{tabular}{llcc}
\toprule
\textbf{Weighting}
  & \textbf{Signal}
  & \textbf{Test}
  & \textbf{Self} \\
\midrule
LinUCB~\cite{li2010contextual}
  & Reward + exploration
  & $0.791 \pm 0.022$
  & $0.791 \pm 0.022$ \\
Weighted Majority~\cite{littlestone1994weighted}
  & Per-class val accuracy
  & $0.849 \pm 0.003$
  & $0.840 \pm 0.003$ \\
Hedge~\cite{freund1997decision}
  & Val loss (exp.\ update)
  & $0.821 \pm 0.020$
  & $0.797 \pm 0.012$ \\
Simplex ERM (Eq.~\ref{eq:deployment-objective})
  & Free weights on val (Adam)
  & $0.812 \pm 0.015$
  & $0.796 \pm 0.012$ \\
\midrule
\rowcolor{ours}
\textbf{Ours (LNTrust)}
  & Relational features
  & $\mathbf{0.865 \pm 0.017}$
  & $\mathbf{0.829 \pm 0.006}$ \\
\bottomrule
\end{tabular}
\end{table}


\section{Conclusion} \label{sec:conclusion}
We introduced LNTrust, a decentralized learning method in which each node
learns a compact trust model over relational features to decide which
neighbors to distill from and how to weight their predictions at
deployment. LNTrust supports heterogeneous architectures, exchanges only queries and soft predictions, and requires no central coordinator. An existing limitation of LNTrust (and decentralized learning in general) is that neighborhoods must contain enough information to make collaboration helpful. Neighborhoods where each node has completely disjoint class distributions will not be aided by this method. Outside the scope of this work, but an important future direction is the analyzing how robust this method is to adversarial actors in this setting. 
\newpage
\bibliographystyle{plain}
\bibliography{neurips}

\begin{thebibliography}{10}

\bibitem{allouah2024fens}
Youssef Allouah, Akash Dhasade, Rachid Guerraoui, Nirupam Gupta, Anne-Marie Kermarrec, Rafael Pinot, Rafael Pires, and Rishi Sharma.
\newblock Revisiting ensembling in one-shot federated learning.
\newblock In {\em Advances in Neural Information Processing Systems}, volume~37, 2024.

\bibitem{bao2023atp}
Wenxuan Bao, Tianxin Wei, Haohan Wang, and Jingrui He.
\newblock Adaptive test-time personalization for federated learning.
\newblock In {\em Advances in Neural Information Processing Systems}, volume~36, 2023.

\bibitem{cisco_fog2015}
Flavio Bonomi, Rodolfo Milito, Jiang Zhu, and Sateesh Addepalli.
\newblock {IoT}: From cloud to fog computing.
\newblock Cisco Blogs, 2015.
\newblock Accessed: 2026-05-03.

\bibitem{chen2021fedbe}
Hong-You Chen and Wei-Lun Chao.
\newblock {FedBE}: Making bayesian model ensemble applicable to federated learning.
\newblock In {\em International Conference on Learning Representations}, 2021.

\bibitem{cisco_edge2024}
{Cisco Systems}.
\newblock What is edge computing?
\newblock Cisco, 2024.
\newblock Accessed: 2026-05-03.

\bibitem{t2020personalized}
Canh~T. Dinh, Nguyen~H. Tran, and Tuan~Dung Nguyen.
\newblock Personalized federated learning with moreau envelopes.
\newblock In {\em Advances in Neural Information Processing Systems}, volume~33, pages 21394--21405, 2020.

\bibitem{fallah2020personalized}
Alireza Fallah, Aryan Mokhtari, and Asuman Ozdaglar.
\newblock Personalized federated learning with theoretical guarantees: {A} model-agnostic meta-learning approach.
\newblock In {\em Advances in Neural Information Processing Systems}, volume~33, pages 3557--3568, 2020.

\bibitem{freund1997decision}
Yoav Freund and Robert~E. Schapire.
\newblock A decision-theoretic generalization of on-line learning and an application to boosting.
\newblock {\em Journal of Computer and System Sciences}, 55(1):119--139, 1997.

\bibitem{hamer2020fedboost}
Jenny Hamer, Mehryar Mohri, and Ananda~Theertha Suresh.
\newblock {FedBoost}: A communication-efficient algorithm for federated learning.
\newblock In {\em Proceedings of the 37th International Conference on Machine Learning}, volume 119 of {\em Proceedings of Machine Learning Research}, pages 3973--3983. PMLR, 2020.

\bibitem{hegedus2021gossip}
István Hegedűs, Gábor Danner, and Márk Jelasity.
\newblock Decentralized learning works: An empirical comparison of gossip learning and federated learning.
\newblock {\em Journal of Parallel and Distributed Computing}, 148:109--124, 2021.

\bibitem{huang2024overcoming}
Chun-Yin Huang, Kartik Srinivas, Xin Zhang, and Xiaoxiao Li.
\newblock Overcoming data and model heterogeneities in decentralized federated learning via synthetic anchors.
\newblock In {\em Proceedings of the 41st International Conference on Machine Learning}, 2024.

\bibitem{itahara2021distillation}
Sohei Itahara, Takayuki Nishio, Yusuke Koda, Masahiro Morikura, and Koji Yamamoto.
\newblock Distillation-based semi-supervised federated learning for communication-efficient collaborative training with non-{IID} private data.
\newblock {\em IEEE Transactions on Mobile Computing}, 22(1):191--205, 2023.

\bibitem{jiang2023fedthe}
Liangze Jiang and Tao Lin.
\newblock Test-time robust personalization for federated learning.
\newblock In {\em International Conference on Learning Representations}, 2023.

\bibitem{kairouz2021advances}
Peter Kairouz, H.~Brendan McMahan, Brendan Avent, Aur{\'e}lien Bellet, Mehdi Bennis, Arjun~Nitin Bhagoji, Kallista Bonawitz, Zachary Charles, Graham Cormode, Rachel Cummings, Rafael G.~L. D'Oliveira, Hubert Eichner, Salim El~Rouayheb, David Evans, Josh Gardner, Zachary Garrett, Adri{\`a} Gasc{\'o}n, Badih Ghazi, Phillip~B. Gibbons, Marco Gruteser, Zaid Harchaoui, Chaoyang He, Lie He, Zhouyuan Huo, Ben Hutchinson, Justin Hsu, Martin Jaggi, Tara Javidi, Gauri Joshi, Mikhail Khodak, Jakub Konec{\v{n}}{\'y}, Aleksandra Korolova, Farinaz Koushanfar, Sanmi Koyejo, Tancr{\`e}de Lepoint, Yang Liu, Prateek Mittal, Mehryar Mohri, Richard Nock, Ayfer {\"O}zg{\"u}r, Rasmus Pagh, Hang Qi, Daniel Ramage, Ramesh Raskar, Mariana Raykova, Dawn Song, Weikang Song, Sebastian~U. Stich, Ziteng Sun, Ananda~Theertha Suresh, Florian Tram{\`e}r, Praneeth Vepakomma, Jianyu Wang, Li~Xiong, Zheng Xu, Qiang Yang, Felix~X. Yu, Han Yu, and Sen Zhao.
\newblock Advances and open problems in federated learning.
\newblock {\em Foundations and Trends in Machine Learning}, 14(1--2):1--210, 2021.

\bibitem{karimireddy2020scaffold}
Sai~Praneeth Karimireddy, Satyen Kale, Mehryar Mohri, Sashank~J. Reddi, Sebastian~U. Stich, and Ananda~Theertha Suresh.
\newblock {SCAFFOLD}: Stochastic controlled averaging for federated learning.
\newblock In {\em Proceedings of the 37th International Conference on Machine Learning}, volume 119 of {\em Proceedings of Machine Learning Research}, pages 5132--5143. PMLR, 2020.

\bibitem{li2019fedmd}
Daliang Li and Junpu Wang.
\newblock {FedMD}: Heterogenous federated learning via model distillation.
\newblock {\em CoRR}, abs/1910.03581, 2019.

\bibitem{li2010contextual}
Lihong Li, Wei Chu, John Langford, and Robert~E. Schapire.
\newblock A contextual-bandit approach to personalized news article recommendation.
\newblock In {\em Proceedings of the 19th International Conference on World Wide Web}, pages 661--670, 2010.

\bibitem{li2021moon}
Qinbin Li, Bingsheng He, and Dawn Song.
\newblock Model-contrastive federated learning.
\newblock In {\em Proceedings of the IEEE/CVF Conference on Computer Vision and Pattern Recognition}, pages 10713--10722, 2021.

\bibitem{li2021ditto}
Tian Li, Shengyuan Hu, Ahmad Beirami, and Virginia Smith.
\newblock Ditto: Fair and robust federated learning through personalization.
\newblock In {\em Proceedings of the 38th International Conference on Machine Learning}, volume 139 of {\em Proceedings of Machine Learning Research}, pages 6357--6368. PMLR, 2021.

\bibitem{li2020fedprox}
Tian Li, Anit~Kumar Sahu, Manzil Zaheer, Maziar Sanjabi, Ameet Talwalkar, and Virginia Smith.
\newblock Federated optimization in heterogeneous networks.
\newblock In {\em Proceedings of Machine Learning and Systems}, volume~2, pages 429--450, 2020.

\bibitem{lian2017can}
Xiangru Lian, Ce~Zhang, Huan Zhang, Cho-Jui Hsieh, Wei Zhang, and Ji~Liu.
\newblock Can decentralized algorithms outperform centralized algorithms? {A} case study for decentralized parallel stochastic gradient descent.
\newblock In {\em Advances in Neural Information Processing Systems}, volume~30, 2018.

\bibitem{lin2020ensemble}
Tao Lin, Lingjing Kong, Sebastian~U. Stich, and Martin Jaggi.
\newblock Ensemble distillation for robust model fusion in federated learning.
\newblock In {\em Advances in Neural Information Processing Systems}, volume~33, pages 2351--2363, 2020.

\bibitem{littlestone1994weighted}
Nick Littlestone and Manfred~K. Warmuth.
\newblock The weighted majority algorithm.
\newblock {\em Information and Computation}, 108(2):212--261, 1994.

\bibitem{marfoq2022knnper}
Othmane Marfoq, Giovanni Neglia, Richard Vidal, and Laetitia Kameni.
\newblock Personalized federated learning through local memorization.
\newblock In {\em Proceedings of the 39th International Conference on Machine Learning}, volume 162 of {\em Proceedings of Machine Learning Research}, pages 15070--15092. PMLR, 2022.

\bibitem{mcmahan2017communication}
H.~Brendan McMahan, Eider Moore, Daniel Ramage, Seth Hampson, and Blaise {Aguera y Arcas}.
\newblock Communication-efficient learning of deep networks from decentralized data.
\newblock In {\em Proceedings of the 20th International Conference on Artificial Intelligence and Statistics}, volume~54 of {\em Proceedings of Machine Learning Research}, pages 1273--1282. PMLR, 2017.

\bibitem{mnih2015humanlevel}
Volodymyr Mnih, Koray Kavukcuoglu, David Silver, Andrei~A. Rusu, Joel Veness, Marc~G. Bellemare, Alex Graves, Martin Riedmiller, Andreas~K. Fidjeland, Georg Ostrovski, Stig Petersen, Charles Beattie, Amir Sadik, Ioannis Antonoglou, Helen King, Dharshan Kumaran, Daan Wierstra, Shane Legg, and Demis Hassabis.
\newblock Human-level control through deep reinforcement learning.
\newblock {\em Nature}, 518(7540):529--533, 2015.

\bibitem{mueller2024fedpae}
Brianna Mueller, W.~Nick Street, Stephen Baek, Qihang Lin, Jingyi Yang, and Yankun Huang.
\newblock {FedPAE}: Peer-adaptive ensemble learning for asynchronous and model-heterogeneous federated learning.
\newblock In {\em Proceedings of the 2024 IEEE International Conference on Big Data}, pages 7961--7970, 2024.

\bibitem{ravikumar2024homogenizing}
Deepak Ravikumar, Gobinda Saha, Sai~Aparna Aketi, and Kaushik Roy.
\newblock Homogenizing non-{IID} datasets via in-distribution knowledge distillation for decentralized learning.
\newblock {\em Transactions on Machine Learning Research}, 2024.

\bibitem{seo2026federated}
Jungwon Seo, Ferhat~Ozgur Catak, Chunming Rong, and Jaeyeon Jang.
\newblock Federated inference: Toward privacy-preserving collaborative and incentivized model serving.
\newblock {\em CoRR}, abs/2603.02214, 2026.

\bibitem{TarvainenV17}
Antti Tarvainen and Harri Valpola.
\newblock Weight-averaged consistency targets improve semi-supervised deep learning results.
\newblock {\em CoRR}, abs/1703.01780, 2017.

\bibitem{yuan2023decdiff}
Lorenzo Valerio, Chiara Boldrini, Andrea Passarella, J{\'a}nos Kert{\'e}sz, M{\'a}rton Karsai, and Gerardo I{\~n}iguez.
\newblock Coordination-free decentralised federated learning in pervasive networks: Overcoming heterogeneity.
\newblock {\em Pervasive and Mobile Computing}, 118:102184, 2026.

\bibitem{zhang2021fedfomo}
Michael Zhang, Karan Sapra, Sanja Fidler, Serena Yeung, and Jose~M. Alvarez.
\newblock Personalized federated learning with first order model optimization.
\newblock In {\em International Conference on Learning Representations}, 2021.

\bibitem{Zhang}
Ying Zhang, Tao Xiang, Timothy~M. Hospedales, and Huchuan Lu.
\newblock Deep mutual learning.
\newblock In {\em Proceedings of the IEEE Conference on Computer Vision and Pattern Recognition}, pages 4320--4328, 2018.

\end{thebibliography}

\newpage 
\newpage

\appendix

\appendix

\newpage
\section{Appendix}
\section{Theory: Proofs and Auxiliary Results}
\label{app:proof}

\subsection{Expressivity of the trust parameterization}
\label{app:expressivity}

We discharge Assumption~\ref{ass:expressivity} by showing that the LNTrust
softmax-MLP class is dense in $\Delta(\bar{\mathcal N}(i))$ under $\ell_1$
whenever the relational features separate distinct neighbors. Together
with the $\ell_1$-Lipschitz continuity of the deployment risk used in the
proof of Proposition~\ref{thm:deploy_oracle}, this collapses the
approximation gap relative to the unrestricted simplex.

For a scalar trust function $g:\mathbb R^d\to\mathbb R$ and features
$\{\phi_{ij}:j\in\bar{\mathcal N}(i)\}$, write
\[
\alpha_{ij}(g)
:=
\frac{\exp(g(\phi_{ij}))}{\sum_{k\in\bar{\mathcal N}(i)} \exp(g(\phi_{ik}))},
\qquad j\in\bar{\mathcal N}(i),
\]
and $\mathcal A_i^{\mathrm{LN}} := \{\alpha_i(g): g\in\mathcal G\}$.

\begin{proposition}[Trust-class density]
\label{prop:trust-expressivity}
Suppose the features $\{\phi_{ij}:j\in\bar{\mathcal N}(i)\}\subset\mathbb R^6$
are pairwise distinct, and let $\mathcal G$ be a function class capable of
interpolating arbitrary real values on this finite set. Then
$\mathcal A_i^{\mathrm{LN}}\supseteq\operatorname{int}\Delta(\bar{\mathcal N}(i))$
and is $\ell_1$-dense in $\Delta(\bar{\mathcal N}(i))$.
\end{proposition}

\begin{proof}
Let $\alpha^\star\in\operatorname{int}\Delta(\bar{\mathcal N}(i))$. The
target logits $z_j := \log\alpha^\star_j$ are finite. By the interpolation
hypothesis on $\mathcal G$ (applicable because the features are pairwise
distinct), there exists $g\in\mathcal G$ with $g(\phi_{ij})=z_j$ for every
$j$. Then
\[
\alpha_{ij}(g)
=
\frac{\exp(z_j)}{\sum_{k}\exp(z_k)}
=
\frac{\alpha^\star_j}{\sum_{k}\alpha^\star_k}
=
\alpha^\star_j,
\]
using $\sum_k\alpha^\star_k=1$. This proves exact representation on the
open simplex.

For an arbitrary $\alpha^\star\in\Delta(\bar{\mathcal N}(i))$, possibly on
the boundary, let $u$ denote the uniform distribution on $\bar{\mathcal N}(i)$
and set $\alpha^\rho := (1-\rho)\alpha^\star + \rho u$ for $\rho\in(0,1)$.
Then $\alpha^\rho\in\operatorname{int}\Delta(\bar{\mathcal N}(i))$, so by
the first part there exists $g_\rho\in\mathcal G$ with
$\alpha_i(g_\rho)=\alpha^\rho$. Since
$\|\alpha^\rho-\alpha^\star\|_1 = \rho\|u-\alpha^\star\|_1 \le 2\rho$,
sending $\rho\to 0$ proves $\ell_1$-density.
\end{proof}

\begin{remark}[Feature collisions]
\label{rmk:collisions}
The pairwise-distinctness hypothesis rules out a trivial obstruction: if
$\phi_{ij}=\phi_{ik}$ for $j\neq k$, any deterministic trust function must
satisfy $g(\phi_{ij})=g(\phi_{ik})$, forcing $\alpha_{ij}=\alpha_{ik}$. The
six relational features include node-specific quantities (validation probe
responses, KL divergence to the local class distribution, neighbor-side
class entropy) that make exact collisions a measure-zero event in practice.
\end{remark}

\begin{remark}[Interpolation capacity of standard ReLU MLPs]
\label{rmk:capacity}
The interpolation hypothesis on $\mathcal G$ holds for shallow ReLU
networks of modest width. Since the features are pairwise distinct, there
exists a direction $a\in\mathbb R^6$ along which the projections
$\{a^\top\phi_{ij}\}$ are pairwise distinct; a one-hidden-layer ReLU
network of width $|\bar{\mathcal N}(i)|$ then interpolates arbitrary values
on these projections. The $6\to 32\to 32\to 1$ trust head used in our
experiments comfortably exceeds the largest closed-neighborhood size in
any setting we evaluate.
\end{remark}

\subsection{Proof of Proposition~\ref{thm:deploy_oracle}}
\label{app:proof_deploy}

\begin{proof}[Proof of Proposition~\ref{thm:deploy_oracle}]
Fix a node $i$ and condition on the deployment predictors
$\theta=(\theta_1,\ldots,\theta_n)$. Write
\[
\bar N_i := \bar{\mathcal N}(i),
\qquad
d_i := |\bar N_i|,
\qquad
S := \mathcal D_i^V = \{z_k\}_{k=1}^{m_i},
\qquad
z_k = (x_k, y_k).
\]
The argument is conditional on fixed predictors: the validation sample
selects deployment weights only after the candidate predictors have been
fixed.

\paragraph{Step 1: Rademacher complexity of the closed-neighborhood simplex
score class.}
For a function family $\mathcal F$, the empirical Rademacher complexity is
$\widehat{\mathfrak R}_S(\mathcal F) = \mathbb E_\sigma[\sup_{f\in\mathcal F}
m_i^{-1}\sum_{k=1}^{m_i}\sigma_k f(z_k)]$, with
$\sigma_1,\ldots,\sigma_{m_i}$ independent Rademacher variables and the
sample $S$ fixed. Define the score class
\[
\mathcal S_i^\Delta
:=
\Bigl\{
z=(x,y)\mapsto \textstyle\sum_{j\in\bar N_i}\alpha_{ij}\,p_j^y(x;\theta_j)
:
\alpha_i\in\Delta(\bar N_i)
\Bigr\}.
\]
For fixed $S$, setting
$c_j(S,\sigma) := m_i^{-1}\sum_{k=1}^{m_i}\sigma_k p_j^{y_k}(x_k;\theta_j)$,
the supremum of $\sum_{j} \alpha_{ij} c_j(S,\sigma)$ over the simplex is
attained at a vertex, so
\[
\widehat{\mathfrak R}_S(\mathcal S_i^\Delta)
= \mathbb E_\sigma\!\Big[\max_{j\in\bar N_i} c_j(S,\sigma)\Big]
= \tfrac{1}{m_i}\,\mathbb E_\sigma\!\Big[\max_{j\in\bar N_i}\langle\sigma, v_j(S)\rangle\Big],
\]
where
$v_j(S) := (p_j^{y_1}(x_1;\theta_j),\ldots,p_j^{y_{m_i}}(x_{m_i};\theta_j))
\in \mathbb R^{m_i}$.
Since $p_j^{y_k}(x_k;\theta_j)\in[0,1]$ we have $\|v_j(S)\|_2\le\sqrt{m_i}$,
and Massart's finite-class lemma gives
\begin{equation}
\widehat{\mathfrak R}_S(\mathcal S_i^\Delta)
\;\le\;
\sqrt{\tfrac{2\log d_i}{m_i}}.
\label{eq:rad-simplex}
\end{equation}

\paragraph{Step 2: Contraction through the clipped log-loss.}
Let $\phi(u) := -\log\max\{u,\varepsilon_{\mathrm{dep}}\}$ for $u\in[0,1]$
and $0<\varepsilon_{\mathrm{dep}}\le 1$. Then
$0 \le \phi(u) \le \log(1/\varepsilon_{\mathrm{dep}})$ and $\phi$ is
$1/\varepsilon_{\mathrm{dep}}$-Lipschitz on $[0,1]$ (constant on
$[0,\varepsilon_{\mathrm{dep}}]$; derivative $1/u\le 1/\varepsilon_{\mathrm{dep}}$
on $[\varepsilon_{\mathrm{dep}},1]$). Define the loss class
\[
\mathcal L_i^\Delta
:=
\Bigl\{
z=(x,y)\mapsto \phi\Big(\textstyle\sum_{j\in\bar N_i}\alpha_{ij}\,p_j^y(x;\theta_j)\Big)
:
\alpha_i\in\Delta(\bar N_i)
\Bigr\}.
\]
Setting $\tilde\phi(u):=\phi(u)-\phi(0)$ leaves $\widehat{\mathfrak R}_S$
unchanged (the constant term contributes
$\mathbb E_\sigma[m_i^{-1}\sum_k \sigma_k]=0$), and $\tilde\phi$ is
$1/\varepsilon_{\mathrm{dep}}$-Lipschitz with $\tilde\phi(0)=0$. The
Ledoux--Talagrand contraction inequality then yields
\[
\widehat{\mathfrak R}_S(\mathcal L_i^\Delta)
\;\le\;
\tfrac{1}{\varepsilon_{\mathrm{dep}}}\,
\widehat{\mathfrak R}_S(\mathcal S_i^\Delta)
\;\le\;
\tfrac{1}{\varepsilon_{\mathrm{dep}}}\sqrt{\tfrac{2\log d_i}{m_i}},
\]
and taking expectations preserves the same bound for
$\mathfrak R_{m_i}(\mathcal L_i^\Delta)$. Restricting to the LNTrust class
$\mathcal L_i^{\mathrm{LN}}$ (defined analogously with
$\alpha_i\in\mathcal A_i^{\mathrm{LN}}$) only shrinks the supremum, so the
same bound applies; no separate complexity analysis of the trust MLP is
required.

\paragraph{Step 3: Uniform deviation and ERM.}
Every loss in $\mathcal L_i^{\mathrm{LN}}$ takes values in
$[0,\log(1/\varepsilon_{\mathrm{dep}})]$. The standard two-sided uniform
Rademacher generalization bound therefore implies that, with probability
at least $1-\delta$ over the draw of $S$,
\[
\sup_{\alpha_i\in\mathcal A_i^{\mathrm{LN}}}
\Bigl|R_{i,\varepsilon}^{\mathrm{dep}}(\theta,\alpha_i)
- \widehat R_{i,\varepsilon}^{V,\mathrm{dep}}(\theta,\alpha_i)\Bigr|
\;\le\;
\Delta_i(\delta),
\]
where
\[
\Delta_i(\delta)
:=
\tfrac{2}{\varepsilon_{\mathrm{dep}}}\sqrt{\tfrac{2\log d_i}{m_i}}
\;+\;
\log\!\tfrac{1}{\varepsilon_{\mathrm{dep}}}\sqrt{\tfrac{\log(2/\delta)}{2m_i}}.
\]
Combining this with the $\xi_i$-suboptimality of $\hat\alpha_i^{\mathrm{dep}}$
yields, on the same event,
\[
R_{i,\varepsilon}^{\mathrm{dep}}(\theta,\hat\alpha_i^{\mathrm{dep}})
\;\le\;
\inf_{\alpha_i\in\mathcal A_i^{\mathrm{LN}}}
R_{i,\varepsilon}^{\mathrm{dep}}(\theta,\alpha_i)
\;+\; \xi_i \;+\; 2\Delta_i(\delta).
\]

\paragraph{Step 4: Closing the infimum gap via density.}
Under Assumption~\ref{ass:expressivity},
Proposition~\ref{prop:trust-expressivity} (Appendix~\ref{app:expressivity})
shows that $\mathcal A_i^{\mathrm{LN}}$ is $\ell_1$-dense in
$\Delta(\bar N_i)$. The risk
$R_{i,\varepsilon}^{\mathrm{dep}}(\theta,\cdot)$ is
$1/\varepsilon_{\mathrm{dep}}$-Lipschitz with respect to $\|\cdot\|_1$:
since each $p_j^y\in[0,1]$,
$|q_i(y\mid x;\theta,\alpha_i)-q_i(y\mid x;\theta,\alpha_i')|
\le \|\alpha_i-\alpha_i'\|_1$, and $\phi$ is
$1/\varepsilon_{\mathrm{dep}}$-Lipschitz; composition then gives the
claim after taking expectations. Density and Lipschitz continuity together
imply
\[
\inf_{\alpha_i\in\mathcal A_i^{\mathrm{LN}}}
R_{i,\varepsilon}^{\mathrm{dep}}(\theta,\alpha_i)
\;=\;
\inf_{\alpha_i\in\Delta(\bar N_i)}
R_{i,\varepsilon}^{\mathrm{dep}}(\theta,\alpha_i).
\]
Substituting and unrolling $\Delta_i(\delta)$ gives
\[
R_{i,\varepsilon}^{\mathrm{dep}}(\theta,\hat\alpha_i^{\mathrm{dep}})
\le
\inf_{\alpha_i\in\Delta(\bar N_i)} R_{i,\varepsilon}^{\mathrm{dep}}(\theta,\alpha_i)
+ \xi_i
+ \tfrac{4}{\varepsilon_{\mathrm{dep}}}\sqrt{\tfrac{2\log|\bar{\mathcal N}(i)|}{m_i}}
+ 2\log\!\tfrac{1}{\varepsilon_{\mathrm{dep}}}\sqrt{\tfrac{\log(2/\delta)}{2m_i}}.
\qedhere
\]
\end{proof}

\subsection{Training as a controlled perturbation of self-only updates}
\label{app:perturbation}

Fix a node $i$. Let $\widehat{\mathcal L}_{i,t}^{\mathrm{sup}}$ denote the
supervised mini-batch loss at round $t$ and
$\widehat{\mathcal L}_{i,t}^{\mathrm{distil}}$ the realized distillation
loss formed from \eqref{eq:ensemble}--\eqref{eq:iw}; if
$\mathcal B_{i,t}^\star=\varnothing$ we set
$\widehat{\mathcal L}_{i,t}^{\mathrm{distil}}\equiv 0$. The realized
LNTrust update is
\begin{align}
\theta_{i,t+\frac12}
&= \theta_{i,t}
- \eta_t^{\mathrm{sup}} \nabla \widehat{\mathcal L}_{i,t}^{\mathrm{sup}}(\theta_{i,t}),
\label{eq:theory-actual-update-sup}\\
\theta_{i,t+1}
&= \theta_{i,t+\frac12}
- \eta_t^{\mathrm{dist}} \lambda_{i,t}^{\mathrm{eff}}
  \nabla \widehat{\mathcal L}_{i,t}^{\mathrm{distil}}(\theta_{i,t+\frac12}).
\label{eq:theory-actual-update}
\end{align}
The self-only baseline $\{\vartheta_{i,t}\}_{t\ge 0}$ shares the same
initialization, supervised mini-batches, and supervised optimizer, but
omits the distillation step:
\begin{equation}
\vartheta_{i,t+1}
= \vartheta_{i,t}
- \eta_t^{\mathrm{sup}} \nabla \widehat{\mathcal L}_{i,t}^{\mathrm{sup}}(\vartheta_{i,t}).
\label{eq:theory-self-update}
\end{equation}

\begin{proposition}[Bounded drift from self-training]
\label{prop:self-stability}
Suppose for each round $t$:
(i) $\widehat{\mathcal L}_{i,t}^{\mathrm{sup}}$ has $\beta_{i,t}$-Lipschitz
gradient; and
(ii) there exists $G_{i,t}\ge 0$ such that
$\|\nabla \widehat{\mathcal L}_{i,t}^{\mathrm{distil}}(\theta)\|\le G_{i,t}$
for every $\theta$ in a region containing both trajectories up to round $t$.
If $\theta_{i,0}=\vartheta_{i,0}$, then for every $T\ge 1$,
\begin{equation}
\|\theta_{i,T}-\vartheta_{i,T}\|
\;\le\;
\exp\!\Big(\textstyle\sum_{s=0}^{T-1}\eta_s^{\mathrm{sup}}\beta_{i,s}\Big)
\textstyle\sum_{t=0}^{T-1}
\eta_t^{\mathrm{dist}}\lambda_{i,t}^{\mathrm{eff}} G_{i,t}.
\label{eq:self-stability-main}
\end{equation}
\end{proposition}

\begin{proof}
Set $\delta_{i,t}:=\theta_{i,t}-\vartheta_{i,t}$ and $\delta_{i,0}=0$.
Subtracting \eqref{eq:theory-self-update} from
\eqref{eq:theory-actual-update-sup} and using assumption~(i),
\[
\|\theta_{i,t+\frac12}-\vartheta_{i,t+1}\|
=
\bigl\|\delta_{i,t}
- \eta_t^{\mathrm{sup}}\bigl(
\nabla\widehat{\mathcal L}_{i,t}^{\mathrm{sup}}(\theta_{i,t})
-\nabla\widehat{\mathcal L}_{i,t}^{\mathrm{sup}}(\vartheta_{i,t})\bigr)\bigr\|
\le \bigl(1+\eta_t^{\mathrm{sup}}\beta_{i,t}\bigr)\|\delta_{i,t}\|.
\]
Adding the distillation step and using assumption~(ii) gives
\[
\|\delta_{i,t+1}\|
\le
\bigl(1+\eta_t^{\mathrm{sup}}\beta_{i,t}\bigr)\|\delta_{i,t}\|
+ \eta_t^{\mathrm{dist}}\lambda_{i,t}^{\mathrm{eff}} G_{i,t}.
\]
Iterating from $\delta_{i,0}=0$ and using $1+a\le e^a$ for $a\ge 0$ to
collapse the product yields \eqref{eq:self-stability-main}.
\end{proof}

\subsection{Gate selectivity}
\label{app:gate}

Proposition~\ref{prop:self-stability} bounds the drift in terms of
$\lambda_{i,t}^{\mathrm{eff}}$. We now show that, at the population level,
the oracle gate is itself monotone in the peer ensemble's excess error on
$P_i$, and that the empirical gate tracks it under validation noise.

Define the population accuracies
\begin{equation}
A_{i,t}^{\mathrm{self}}
:= \textstyle\sum_{c=1}^C w_i^c\,\mathbb P(\hat y_{i,t}(X)=c\mid Y=c),
\qquad
A_{i,t}^{\mathrm{peer}}
:= \textstyle\sum_{c=1}^C w_i^c\,\mathbb P(\hat y_{i,t}^{\mathrm{ens}}(X)=c\mid Y=c),
\label{eq:pop-acc}
\end{equation}
with $(X,Y)\sim P_i$. The peer excess risk is
\begin{equation}
\delta_{i,t}^{\mathrm{peer}}
:= (1-A_{i,t}^{\mathrm{peer}})-(1-A_{i,t}^{\mathrm{self}})
= A_{i,t}^{\mathrm{self}}-A_{i,t}^{\mathrm{peer}},
\label{eq:peer-excess-risk}
\end{equation}
and the oracle gate is
\begin{equation}
\bar\lambda_{i,t}
:= \lambda_{\mathrm{distil}}\cdot
\min\!\Big(1,\ \tfrac{A_{i,t}^{\mathrm{peer}}}{A_{i,t}^{\mathrm{self}}+\epsilon}\Big).
\label{eq:oracle-gate}
\end{equation}

\begin{lemma}[Oracle gate selectivity]
\label{lem:gate_selectivity}
For fixed $A_{i,t}^{\mathrm{self}}$, the map
$\delta_{i,t}^{\mathrm{peer}}\mapsto \bar\lambda_{i,t}$ is monotone
nonincreasing.
\end{lemma}

\begin{proof}
Substituting
$A_{i,t}^{\mathrm{peer}}=A_{i,t}^{\mathrm{self}}-\delta_{i,t}^{\mathrm{peer}}$
from \eqref{eq:peer-excess-risk} into \eqref{eq:oracle-gate} makes
$\bar\lambda_{i,t}$ a nonincreasing function of
$\delta_{i,t}^{\mathrm{peer}}$ for fixed $A_{i,t}^{\mathrm{self}}$.
\end{proof}

Let $\widehat A_{i,t}^{\mathrm{self}}$ and $\widehat A_{i,t}^{\mathrm{peer}}$
denote the empirical weighted accuracies on
$\mathcal D_i^{\mathrm{val}}$, with corresponding empirical gate
\begin{equation}
\widehat\lambda_{i,t}
:= \lambda_{\mathrm{distil}}\cdot
\min\!\Big(1,\ \tfrac{\widehat A_{i,t}^{\mathrm{peer}}}{\widehat A_{i,t}^{\mathrm{self}}+\epsilon}\Big).
\label{eq:emp-gate}
\end{equation}

\begin{proposition}[Empirical gate tracks oracle gate]
\label{prop:gate_empirical}
Suppose $\widehat A_{i,t}^{\mathrm{self}}$ and
$\widehat A_{i,t}^{\mathrm{peer}}$ are unbiased empirical averages of
$[0,1]$-valued random variables over
$m_i := |\mathcal D_i^{\mathrm{val}}|$ examples. Then
\[
\mathbb E\bigl[|\widehat\lambda_{i,t}-\bar\lambda_{i,t}|\bigr]
\;\le\;
L_\epsilon
\Bigl(
\mathbb E\bigl[|\widehat A_{i,t}^{\mathrm{self}}-A_{i,t}^{\mathrm{self}}|\bigr]
+ \mathbb E\bigl[|\widehat A_{i,t}^{\mathrm{peer}}-A_{i,t}^{\mathrm{peer}}|\bigr]
\Bigr),
\qquad
L_\epsilon = \lambda_{\mathrm{distil}}\bigl(\tfrac{1}{\epsilon}+\tfrac{1}{\epsilon^2}\bigr).
\]
\end{proposition}

\begin{proof}
The map $f(u,v) := \lambda_{\mathrm{distil}}\min(1, u/(v+\epsilon))$ on
$[0,1]^2$ is Lipschitz: $\partial_u f \le \lambda_{\mathrm{distil}}/\epsilon$
and $|\partial_v f|\le \lambda_{\mathrm{distil}}/\epsilon^2$ on the region
where the $\min$ is active, so the joint Lipschitz constant is bounded by
$L_\epsilon$. Apply the bound to
$f(\widehat A^{\mathrm{peer}},\widehat A^{\mathrm{self}})$ versus
$f(A^{\mathrm{peer}},A^{\mathrm{self}})$ and take expectations.
\end{proof}

\subsection{Gradient bound for the KL distillation loss}
\label{app:grad_bound}

\begin{proposition}[Distillation gradient bound]
\label{prop:distil-grad}
Fix a round $t$ and suppose $\|\nabla_{\theta_i} h_{\theta_i}(x)\|_{\mathrm{op}}
\le M_{i,t}$ for every $x\in\mathcal B_{i,t}^\star$ on the parameter region
visited at round $t$. Then
$\bigl\|\nabla\widehat{\mathcal L}_{i,t}^{\mathrm{distil}}(\theta_i)\bigr\|
\le C\sqrt{2}\, M_{i,t}$;
the left-hand side is zero if $\mathcal B_{i,t}^\star=\varnothing$.
\end{proposition}

\begin{proof}
Write $s_{\theta_i}(x) := \mathrm{softmax}(h_{\theta_i}(x))$ and
$q_{i,t}(x) := \bar p_i(x)$. For one retained example
$x\in\mathcal B_{i,t}^\star$,
$\ell_x(\theta_i)
= \omega_i(x)\,\mathrm{KL}(q_{i,t}(x)\,\|\,s_{\theta_i}(x))$
has gradient
$\nabla_{\theta_i}\ell_x(\theta_i)
= \omega_i(x)\,J_{\theta_i}(x)^\top\bigl(s_{\theta_i}(x)-q_{i,t}(x)\bigr)$,
where $J_{\theta_i}(x):=\nabla_{\theta_i}h_{\theta_i}(x)$. Both $s_{\theta_i}(x)$
and $q_{i,t}(x)$ lie in the simplex, so
$\|s_{\theta_i}(x)-q_{i,t}(x)\|_2\le\sqrt{2}$; and
$\omega_i(x)=C\,w_i^{\hat c(x)}\le C$ by~\eqref{eq:iw}. Hence
$\|\nabla_{\theta_i}\ell_x(\theta_i)\|\le C\sqrt{2}\,M_{i,t}$. The bound
is preserved by averaging over $\mathcal B_{i,t}^\star$.
\end{proof}

\section{Additional Experiments and Ablations}
\label{app:additional_experiments}

\paragraph{Validation set size.}
Table~\ref{tab:ablation_val} varies the fraction of each node's labeled
data reserved for validation (trust estimation and deployment gating)
versus training.
Larger validation sets improve trust estimation (the deploy gain rises
from $+3.5$ pp at $5\%$ to $+9.5$ pp at $20\%$), but reduce the
training set, lowering Stage~1 accuracy.
At $5\%$, nodes have only ${\sim}50$ validation examples on average,
too few for reliable probe-based trust features; the trust model
defaults to high self-weight and the deployment ensemble barely
improves over self.
The optimum is $20\%$: the trust model has sufficient validation
signal to learn discriminative neighbor weights, and the training set
remains large enough for strong Stage~1 models.
At $30\%$, the marginal trust improvement no longer compensates for
the reduced training data.

\begin{table}[htbp]
\centering
\caption{
  Validation fraction ablation (sparse heterogeneous, CIFAR-10,
  BA $m{=}2$, $p{=}0.1$, mixed architectures, 3 seeds).
  Deploy gain = Test $-$ Stage~1.
}
\label{tab:ablation_val}
\setlength{\tabcolsep}{4pt}
\begin{tabular}{rccc}
\toprule
\textbf{Val frac.} & \textbf{Test} & \textbf{Self} & \textbf{Deploy gain} \\
\midrule
$0.05$
  & $0.822 \pm 0.012$ & $0.811 \pm 0.008$ & $+0.035$ \\
$0.10$
  & $0.841 \pm 0.020$ & $0.819 \pm 0.012$ & $+0.057$ \\
\rowcolor{ours}
$0.20$
  & $\mathbf{0.865 \pm 0.017}$ & $\mathbf{0.829 \pm 0.006}$ & $\mathbf{+0.095}$ \\
$0.30$
  & $0.855 \pm 0.018$ & $0.825 \pm 0.005$ & $+0.090$ \\
\bottomrule
\end{tabular}
\end{table}

\paragraph{Trust feature importance.}
Figure~\ref{fig:feat_imp} reports permutation importance for each of
the six trust features, measured as the increase in validation cross-entropy
when that feature is randomly shuffled across candidate neighbors
(averaged over all nodes and 10 repetitions per seed).

The six features cluster into three functional groups.
\emph{Direct quality signals} (weighted probe accuracy and unweighted
probe accuracy) are computed from each neighbor's predictions on the
evaluating node's validation set with ground-truth labels, and
together account for the largest share of importance.
Weighted probe accuracy ($\sum_c w_i^c \rho_{ij}^c$) is the single
most informative feature because it measures neighbor quality on
exactly the classes the evaluating node cares about, rather than
treating all classes equally.
\emph{Distribution inference signals} (KL divergence, overlap, and
neighbor specialization entropy) are inferred from neighbor
predictions on the shared IID pool without access to labels.
KL divergence ($\mathrm{KL}(w_i \| \hat w_j)$) is the second most important
feature overall, confirming that distribution mismatch is a strong
predictor of neighbor utility even when measured indirectly.
\emph{Structural signals} (neighbor degree $d_j/n$) capture hub
connectivity: high-degree neighbors have absorbed broader knowledge
through distillation and are more likely to be useful teachers.

\begin{figure}[htbp]
\centering
\includegraphics[width=\linewidth]{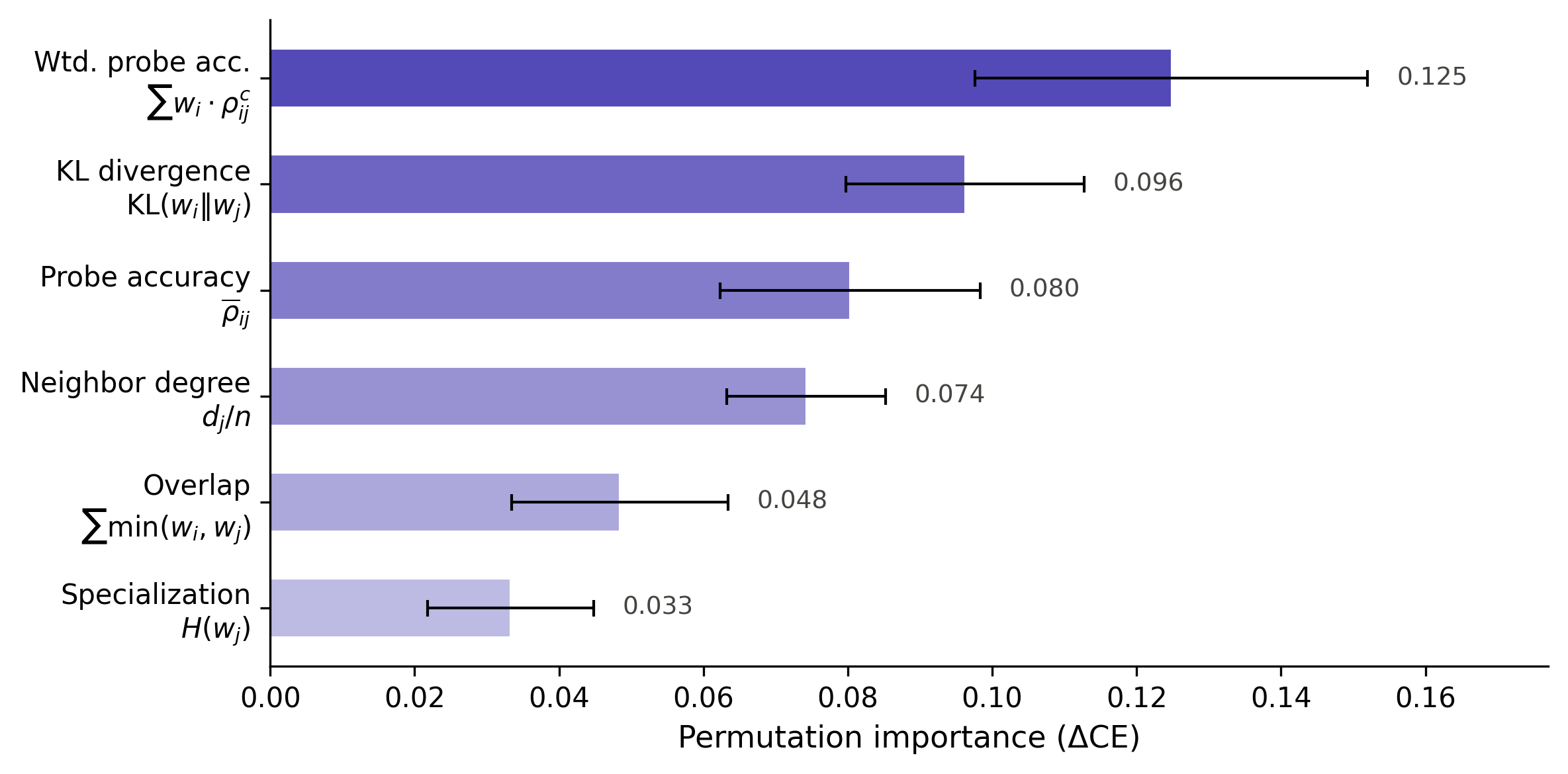}
\caption{
  Permutation importance of the six trust features (sparse
  heterogeneous setting, CIFAR-10, seed~0, round~200).
  Bars show mean $\Delta$CE when each feature is shuffled across
  candidate neighbors; error bars show standard deviation over 10
  repetitions.
  Direct quality signals (probe accuracy variants) and distribution
  inference (KL divergence) dominate, with structural information
  (neighbor degree) providing complementary signal.
}
\label{fig:feat_imp}
\end{figure}

\smallskip
\noindent\textbf{Pseudo-label budget ablation.}
Table~\ref{tab:ablation_pseudo} decomposes the two LNTrust mechanisms.
With $B{=}0$ pseudo-labels, deployment-time ensembling alone reaches
$0.843$, a $5.2$-point gain over Independent Learning ($0.791$) even
though the per-node classifier (\textbf{Self}$=0.777$) is weaker than
the independent baseline, a direct consequence of withholding labeled
examples for validation, the cost LNTrust pays to enable the ensemble.
The deployment ensemble more than recoups this cost without any unlabeled
data or distillation. Distillation then compounds the gain non-monotonically:
$B$ from $0{\to}500{\to}1000$ lifts test accuracy to $0.865$ and Self to
$0.829$, but $B{=}1500$ reverses both, consistent with peer pseudo-labels
overwriting locally informative signal past a saturation point. The
$^\dagger$ row, which removes the unlabeled pool and reduces trust features
to probe accuracy and neighbor degree, matches the $B{=}0$ row at $0.842$:
the pool's role is specifically to power distillation, and LNTrust's
deployment component runs unchanged in settings where no shared unlabeled
data is available.

\begin{table}[htbp]
\centering
\caption{
  Distillation budget ablation (sparse heterogeneous, CIFAR-10,
  BA $m{=}2$, $p{=}0.1$, mixed architectures, 3 seeds).
  Independent Learning (no communication) shown for reference.
  $^\dagger$Unlabeled pool removed entirely; trust features rely
  only on probe accuracy and neighbor degree.
}
\label{tab:ablation_pseudo}
\setlength{\tabcolsep}{4pt}
\begin{tabular}{rcc}
\toprule
\textbf{Pseudo/rnd} & \textbf{Test} & \textbf{Self} \\
\midrule
\multicolumn{3}{l}{\textit{No collaboration}} \\
---
  & $0.791 \pm 0.022$ & $0.791 \pm 0.022$ \\
\midrule
\multicolumn{3}{l}{\textit{LNTrust (trust ensemble, varying distillation)}} \\
0
  & $0.843 \pm 0.023$ & $0.777 \pm 0.012$ \\
500
  & $0.853 \pm 0.020$ & $0.822 \pm 0.007$ \\
\rowcolor{ours}
1000
  & $\mathbf{0.865 \pm 0.017}$ & $\mathbf{0.829 \pm 0.006}$ \\
1500
  & $0.855 \pm 0.023$ & $0.820 \pm 0.004$ \\
\midrule
\multicolumn{3}{l}{\textit{LNTrust (no unlabeled pool, no distillation)}} \\
0$^\dagger$
  & $0.842 \pm 0.026$ & $0.779 \pm 0.007$ \\
\bottomrule
\end{tabular}
\end{table}

\paragraph{Trust update frequency.}
Table~\ref{tab:ablation_freq} varies the trust update frequency $F$:
the number of distillation rounds between consecutive
reprobe-and-retrain cycles for the per-node trust models.
All runs in this ablation use $\text{val\_fraction}{=}0.05$ for
internal comparability; the main results
(Tables~\ref{tab:agnostic_skewed}--\ref{tab:dense_homo}) use $F{=}10$
with $\text{val\_fraction}{=}0.20$.

Performance is stable across $F \in \{1, 5, 10, 15\}$, with no
significant differences in deployed test accuracy ($0.823$--$0.831$).
$F{=}10$ achieves the highest test accuracy, though the margin over
$F{=}5$ is within one standard deviation.
This robustness is notable because $F$ controls the staleness of the
trust model's training signal: at $F{=}15$ the trust model operates
on probe responses that are 15 rounds old, yet achieves the same
accuracy as $F{=}1$.
At the other extreme, $F{=}1$ retrains the trust model every round,
creating a tight feedback loop between trust weights and model
updates.
Each round, the trust model is retrained on probe responses from
models that were themselves shaped by the previous trust weights'
distillation signal.
This circular dependency is structurally analogous to the instability
that motivates target networks in DQN~\cite{mnih2015humanlevel}: in
both cases, the optimization target (Q-values in DQN, trust-weighted
ensemble quality here) co-evolves with the model being trained, and
updating both simultaneously can produce oscillatory dynamics.
In our setting, however, the trust gate provides a natural
stabilizer: even if trust weights oscillate between reprobe cycles,
the distribution-weighted gate comparison suppresses distillation
whenever the ensemble degrades, preventing the feedback loop from
amplifying harmful updates.
Combined with the low capacity of the trust MLP (6 features, 32
hidden units), this limits the potential for overfitting to
transient probe signals.

The practical implication is that $F$ can be chosen primarily on
communication grounds, since higher $F$ amortizes probe communication
cost over more rounds without sacrificing accuracy.

\begin{table}[htbp]
\centering
\caption{
  Trust update frequency ablation (sparse heterogeneous, CIFAR-10,
  BA $m{=}2$, $p{=}0.1$, mixed architectures, 3 seeds,
  $\text{val\_fraction}{=}0.05$).
  $F$: number of distillation rounds between trust reprobe/retrain
  cycles.
  Stage~1 accuracy is $0.787 \pm 0.013$ for all rows
  (shared pretrained models).
}
\label{tab:ablation_freq}
\setlength{\tabcolsep}{4pt}
\begin{tabular}{rcc}
\toprule
$F$ & \textbf{Test} & \textbf{Self} \\
\midrule
$1$
  & $0.824 \pm 0.011$ & $0.810 \pm 0.010$ \\
$5$
  & $0.825 \pm 0.009$ & $0.815 \pm 0.011$ \\
\rowcolor{ours}
$10$
  & $\mathbf{0.831 \pm 0.009}$ & $\mathbf{0.815 \pm 0.009}$ \\
$15$
  & $0.823 \pm 0.014$ & $0.813 \pm 0.013$ \\
\bottomrule
\end{tabular}
\end{table}

\paragraph{Trust gate.}
Table~\ref{tab:ablation_gate} ablates the trust-gated distillation
mechanism. The gate $\delta_i$
(Eq.~\ref{eq:emp-gate}) compares the trust-weighted ensemble loss to
the local model's loss on $\mathcal D_i^{\mathrm{val}}$ and suppresses
the distillation step whenever the ensemble would degrade node $i$'s
predictor; the rest of the LNTrust pipeline (per-node trust MLP,
six-feature representation, hedge weighting, Stage~2 schedule) is
unchanged in both rows. Both runs share the identical Stage~1
checkpoint per seed, so any difference is attributable to Stage~2
distillation behavior.

In this low-validation regime ($\text{val\_fraction}{=}0.05$),
removing the gate \emph{improves} mean Test accuracy by roughly
1\,pp ($0.835$ vs $0.825$). The Self score moves only marginally
($0.814$ vs $0.812$), so the gain is concentrated in deployment-time
selection rather than in the trained models themselves. We read this
as the gate occasionally vetoing distillation steps that would in
fact have helped, which is the cost of trading variance for safety:
with only $\approx 50$ validation examples per node, the empirical
gate
$\hat\delta_i = \widehat L_{\mathrm{val}}(\bar h^{\mathrm{ens}}_i)
- \widehat L_{\mathrm{val}}(h_{\theta_i})$
is itself noisy and can flip the wrong way for individual rounds.
Proposition~\ref{prop:gate_empirical} bounds the magnitude of these flips
in terms of $|\mathcal D_i^{\mathrm{val}}|$, and the bound is
loosest in exactly this regime.

The gate's value is therefore not a mean-test-accuracy gain but a
bounded-drift safety property. With the gate active, the
controlled-perturbation result of
Proposition~\ref{prop:self-stability} bounds how far Stage~2 can
push each node's parameters from the self-only trajectory it would
have followed without distillation, and the population gate
underlying that bound is monotone nonincreasing in the peer
ensemble's excess risk on $\mu_i$
(Lemma~\ref{lem:gate_selectivity}); the empirical gate tracks the
oracle gate up to validation-noise terms
(Proposition~\ref{prop:gate_empirical}). This is the property we want
when LNTrust is deployed in settings where some neighbors may be
misconfigured, mismatched in distribution, or worse: bounded drift
from local supervised learning is a guarantee about the worst case,
not about the mean. In the benign per-node-skewed CIFAR-10
benchmark used here, none of the 50 nodes is adversarial and most
ensembles are well-calibrated, so the gate's protection is
overinsured. We retain the gate by default in our headline results
because the cost is small (about 1\,pp) and the bounded-drift
property matters for any deployment outside this controlled
benchmark.

\begin{table}[htbp]
\centering
\caption{
  Trust-gate ablation (sparse heterogeneous, CIFAR-10,
  BA $m{=}2$, $p{=}0.1$, mixed architectures, 3 seeds,
  $\text{val\_fraction}{=}0.05$).
  Single-flag difference between rows: \texttt{--trust\_gate} on
  vs.\ off; all other hyperparameters and the Stage~1 checkpoint
  are identical.
  Stage~1 accuracy is $0.787 \pm 0.013$ for both rows
  (shared pretrained models).
}
\label{tab:ablation_gate}
\setlength{\tabcolsep}{4pt}
\begin{tabular}{lcc}
\toprule
\textbf{Configuration} & \textbf{Test} & \textbf{Self} \\
\midrule
\rowcolor{ours}
Gate on (default)
  & $0.825 \pm 0.007$ & $0.812 \pm 0.009$ \\
Gate off
  & $\mathbf{0.835 \pm 0.009}$ & $\mathbf{0.814 \pm 0.006}$ \\
\bottomrule
\end{tabular}
\end{table}

\paragraph{Confidence filter threshold.}
Table~\ref{tab:ablation_tau} varies $\tau_{\mathrm{conf}}$, the margin
above the uniform level $1/C$ used to filter low-confidence pseudo-labels
during distillation. Performance is stable across all tested values,
with no significant differences in deployed test accuracy. This confirms
that the filter's primary role is computational (discarding uninformative
examples) rather than a sensitive tuning decision, and that the default
$\tau_{\mathrm{conf}}{=}0.2$ is a conservative but safe choice.

\begin{table}[htbp]
\centering
\caption{
  Confidence filter ablation (sparse heterogeneous, CIFAR-10,
  BA $m{=}2$, $p{=}0.1$, mixed architectures, $n{=}100$, 3 seeds).
  $\tau_{\mathrm{conf}}{=}0$ disables the filter entirely.
}
\label{tab:ablation_tau}
\setlength{\tabcolsep}{4pt}
\begin{tabular}{rcc}
\toprule
$\tau_{\mathrm{conf}}$ & \textbf{Test} & \textbf{Self} \\
\midrule
$0.0$    & $0.811 \pm 0.007$ & $0.791 \pm 0.009$ \\
$0.1$               & $0.811 \pm 0.006$ & $0.796 \pm 0.007$ \\
\rowcolor{ours}
$0.2$     & $\mathbf{0.812 \pm 0.009}$ & $\mathbf{0.795 \pm 0.006}$ \\
\bottomrule
\end{tabular}
\end{table}

\section{Hyperparameters}
\label{app:hyperparams}

This section reports the exact configuration used for every LNTrust
result in the main paper. Hyperparameters are split by role: Table~%
\ref{tab:hp_env} lists \emph{environmental} parameters that define each
experimental scenario (dataset, graph topology, label skew, number of
nodes), while Table~\ref{tab:hp_learn} lists \emph{learning}
parameters that control optimization, distillation, and the trust
mechanism. Most learning parameters are shared across all four
settings: only the few that genuinely differ are repeated per column.
Throughout this section we abbreviate \textbf{Sp.\ Het.\ C10} =
sparse heterogeneous CIFAR-10 (Table~\ref{tab:agnostic_skewed}),
\textbf{Sp.\ Het.\ C100} = sparse heterogeneous CIFAR-100
(Table~\ref{tab:agnostic_skewed}),
\textbf{Geo.\ ES} = geographic EuroSAT
(Table~\ref{tab:eurosat_geo}), and
\textbf{Dense Hom.} = dense homogeneous CIFAR-10
(Table~\ref{tab:dense_homo}).  We fix
$\tau_{\mathrm{abs}}{=}0.2$ and $\tau_{\mathrm{conf}}{=}0.1$ in all
experiments, so the active threshold is $0.2$ for both CIFAR-10
($1/C{+}\tau_{\mathrm{conf}}{=}0.2$, dominated by neither term) and
CIFAR-100 ($\tau_{\mathrm{abs}}$ dominates over $1/C{+}\tau_{\mathrm{conf}}{=}0.11$).

\begin{table}[htbp]
\centering
\caption{
  Environmental hyperparameters: scenario-defining settings that
  determine the dataset, graph, and node distribution for each
  experiment. Number of nodes ($n{=}50$), number of seeds ($3$), and
  the held-out test fraction per node ($0.15$) are shared across all
  four scenarios. Architecture mix is 90\% MobileNetV2 + 10\%
  EfficientNet-B0 (degree-ordered, hubs get EfficientNet) for the
  three heterogeneous-architecture settings; the dense homogeneous
  setting uses MobileNetV2 only.
}
\label{tab:hp_env}
\setlength{\tabcolsep}{4pt}
\begin{tabular}{lcccc}
\toprule
\textbf{Parameter}
  & \textbf{Sp.\ Het.\ C10}
  & \textbf{Sp.\ Het.\ C100}
  & \textbf{Geo.\ ES}
  & \textbf{Dense Hom.} \\
\midrule
Dataset                       & CIFAR-10        & CIFAR-100       & EuroSAT         & CIFAR-10 \\
\# classes                    & 10              & 100             & 10              & 10 \\
Graph model                   & BA $m{=}2$      & BA $m{=}2$      & BA $m{=}2$ (geo)& Erd\H{o}s--R\'enyi \\
Edge density $p$              & $0.10$          & $0.10$          & $0.10$          & $0.50$ \\
Architecture mix              & 90/10 mob/efn   & 90/10 mob/efn   & 90/10 mob/efn   & MobileNetV2 only \\
Heterogeneity source          & synthetic skew  & synthetic skew  & geographic      & synthetic skew \\
Skew factor                   & $10$            & $5$             & $-$ (data)      & $10$ \\
Min.\ classes / node          & $2$             & $10$            & $2$             & $2$ \\
Per-node sample size          & $1000$          & $1000$          & $1000$          & $1000$ \\
\bottomrule
\end{tabular}
\end{table}

\begin{table}[htbp]
\centering
\caption{
  Learning hyperparameters for LNTrust. Values shown in the
  ``Common'' column are identical across all four experimental
  scenarios; only parameters that genuinely differ get scenario-
  specific columns. The trust block (lower group) is identical in
  every run, indicating that the trust mechanism does not require
  per-dataset retuning. The pseudo-label weight, soft distillation
  alpha, and supervised steps per round are also shared. Only the
  soft-distillation toggle and a small number of optional regularizers
  vary by scenario.
}
\label{tab:hp_learn}
\setlength{\tabcolsep}{4pt}
\begin{tabular}{lccccc}
\toprule
\textbf{Parameter}
  & \textbf{Common}
  & \textbf{Sp.\ Het.\ C10}
  & \textbf{Sp.\ Het.\ C100}
  & \textbf{Geo.\ ES}
  & \textbf{Dense Hom.} \\
\midrule
\multicolumn{6}{l}{\textit{Optimization}} \\
Stage~1 max rounds            & $50$         &              &              &              &              \\
Stage~2 max rounds            & $200$        &              &              &              &              \\
Sup.\ steps / node / round    & $5$          &              &              &              &              \\
Pseudo warm-up rounds         & $5$          &              &              &              &              \\
Pseudo examples / round       & $1000$       &              &              &              &              \\
Pseudo-label weight           & $-$          & $0.4$        & $0.5$        & $0.5$        & $0.5$        \\
Soft distillation             & $-$          & off          & on           & off          & on           \\
Soft-distil $\alpha$          & $0.3$        &              &              &              &              \\
Head dropout $p$              & $0.5$        &              &              &              &              \\
Validation fraction           & $0.20$       &              &              &              &              \\
\midrule
\multicolumn{6}{l}{\textit{Trust mechanism}} \\
Trust hidden units            & $32$         &              &              &              &              \\
Trust LR                      & $0.01$       &              &              &              &              \\
Trust optimizer steps         & $200$        &              &              &              &              \\
Trust update freq.\ ($F$)     & $10$         &              &              &              &              \\
Trust deployment gate         & $-$          & on           & on           & on           & off          \\
Val.\ optimizer steps         & $300$        &              &              &              &              \\
Val.\ optimizer $\tau$        & $0.01$       &              &              &              &              \\
\bottomrule
\end{tabular}
\end{table}

\paragraph{What the tables make explicit.}
The trust-mechanism block is constant across all four scenarios:
trust hidden width, trust learning rate, trust optimizer step count,
and trust update frequency $F$ are unchanged from CIFAR-10 to
CIFAR-100 to EuroSAT to the dense homogeneous setting. We did not
retune the trust head per dataset. The same observation applies to
the supervised-step budget, the pseudo-label weight, and the
soft-distillation $\alpha$. The only learning parameters that change
between scenarios are two small qualitative toggles
(\textit{soft distillation} and \textit{trust deployment gate}),
selected once per scenario and held fixed across seeds. All
environmental parameters in Table~\ref{tab:hp_env} are properties of
the experimental setup (graph density, dataset, label skew) rather
than tunable knobs of the method.


\section{Compute Resources}
\label{app:compute}

All experiments were run on a single-GPU-per-run basis on an internal
cluster. DESA baselines were trained on NVIDIA RTX A6000 GPUs (48\,GB)
because their anchor-synthesis step has a larger memory footprint; all
other methods (LNTrust, DML, FedPAE, Mean Teacher, Independent
Learning, D-PSGD, Gossip-FedAvg, DecDiff-VT) were trained on NVIDIA A40
GPUs (48\,GB). A full 200-round simulation of $n{=}50$ nodes averaged
over 3 seeds takes approximately 5 minutes on a single A40. Each
reported table entry is the mean over these 3 seeds.

For experimental efficiency, we freeze backbones (MobileNetV2 and
EfficientNet-B0) at ImageNet-pretrained weights and precompute features
once per dataset and architecture, caching them to disk; each node then
trains only its classifier head during simulation. This is an
experimental choice that applies uniformly to all methods and does not
affect the protocol or the method itself, which is defined independently
of whether backbones are trained or frozen. The precomputation script is
included in the released code to support reproduction.


\section{Broader Impacts}
\label{app:broader_impacts}

LNTrust studies how nodes in a decentralized system can collaborate
without sharing parameters, gradients, or raw data. Both the positive
and negative societal impacts follow from this protocol rather than
from the specific trust-estimation algorithm.

\paragraph{Potential positive impacts.}
The method enables useful collaboration in settings where parameter
sharing is infeasible due to architectural heterogeneity,
communication constraints, or institutional boundaries that forbid
weight transfer. Examples include cross-institutional medical
modeling, edge-device learning across manufacturers, and
geographically distributed sensing where nodes run different hardware
or operate under different data-governance regimes. In these
settings, a trust-weighted output-only protocol lets each node
improve without requiring a common model architecture or a central
coordinator, which can lower the practical barrier to collaboration.

\paragraph{Potential negative impacts.}
Three risks deserve explicit mention. First, repeated black-box
querying reveals information about a neighbor's predictor through its
outputs; an adversarial querier could plausibly use probe responses
to reconstruct aspects of a neighbor's decision boundary or training
distribution, so any deployment should account for query-based
leakage and consider rate limits or privacy-preserving response
mechanisms. Second, the same protocol that enables beneficial
collaboration could be applied to distributed surveillance or
automated decision systems whose outputs affect people; the technical
contribution itself is neutral to application domain, and we do not
address those deployment contexts here. Third, LNTrust concentrates
weight on the neighbors whose predictions best match each node's
local validation distribution, which can leave
minority-distribution nodes under-served when no well-aligned
neighbor exists on the graph. This is a fairness consideration that
matters whenever the graph-wide distribution is itself skewed.

These risks motivate cautious deployment and further analysis
(formal privacy accounting for repeated queries, robustness to
adversarial or faulty neighbors, and per-node fairness auditing)
before use in high-stakes settings.

\section{Limitations}
\label{app:limitations}

The main paper notes limitations briefly at the end of Section~\ref{sec:conclusion}. Here
we expand on them.

\paragraph{Protocol assumptions.}
LNTrust assumes synchronous communication rounds and an undirected,
static graph. Asynchronous updates, churn (nodes joining or leaving
mid-training), and directional communication links are not modeled,
and the controlled-perturbation bound in
Proposition~\ref{prop:self-stability} assumes supervised and
distillation steps alternate in the order specified in
\eqref{eq:theory-actual-update}. The protocol also assumes that
neighbors respond truthfully to queries; we do not analyze Byzantine
neighbors that return adversarially crafted soft predictions. We generally assume the neighborhoods are not sufficiently degenerate with respect to class coverage. If a node does not have any neighbor that can collaborate; collaboration does not help.

\paragraph{Validation-data requirement.}
Each node must withhold a portion of its labeled budget for
validation, and the trust model's quality degrades when that budget
is small (Table~\ref{tab:ablation_val}). Nodes with very few labeled
examples may not be able to afford the split at all, in which case
the deployment gate reduces to near-uniform self-weight and the
method provides little benefit over independent learning. The
$O(\sqrt{\log|\bar{\mathcal N}(i)|/m_i})$ rate in
Proposition~\ref{thm:deploy_oracle} makes this dependence explicit.

\paragraph{Inference-time cost.}
Deployment requires each node to query all neighbors in its closed
neighborhood at inference time and combine their soft predictions.
For latency-sensitive applications this may be prohibitive, and we do
not study approximations (e.g., top-$k$ trust pruning at deployment,
or amortizing queries across batches) beyond the single
\texttt{top\_k\_teachers} ablation reported in the code release. The
top-$k$ option is available but off by default in our headline
results.

\paragraph{Experimental scope.}
Our experiments cover three vision datasets (CIFAR-10, CIFAR-100,
EuroSAT), two architectures (MobileNetV2 and EfficientNet-B0), three
graph families (sparse BA, dense Erd\H{o}s--R\'enyi, geographic BA),
and $n{=}50$ nodes. We include an ablation where $n=100$, but do not evaluate at larger scales ($n$ in thousands), on non-image modalities (text, tabular,
time-series), or with architectures that cannot share a common
output space (e.g., different label sets across nodes). The
backbone-caching optimization described in
Appendix~\ref{app:compute} means our runs also do not stress-test
end-to-end backbone training; we expect the qualitative conclusions
to carry over but have not verified this empirically.

\paragraph{Theory scope.}
The theoretical guarantees are local, not global. Proposition~%
\ref{thm:deploy_oracle} bounds deployment-weight estimation error
for fixed predictors, and
Propositions~\ref{prop:self-stability}--\ref{prop:distil-grad}
together bound how far Stage~2 drifts from self-training. Neither
result certifies that collaboration improves over self-training
end-to-end; that question is settled empirically in
Section~\ref{sec:experiments}, not proved.

\paragraph{Privacy and robustness.}
We do not provide formal privacy guarantees for the query protocol.
Soft predictions leak information about a neighbor's decision
boundary and training distribution, and repeated querying likely
compounds this leakage, but we do not quantify it via differential
privacy accounting or membership-inference analysis. Similarly,
robustness to adversarial neighbors (who may return crafted logits
to manipulate trust weights) is outside the scope of this work.
Both are natural directions for follow-up.

\section{Topology Visualizations}
\label{app:topology}

This section visualizes the two other graph topologies used in the main
experiments (the third is given in Figure 2. Each figure shows seed~0; degree statistics reported in the
main text are averaged over 3 seeds.

\paragraph{Sparse heterogeneous topology (CIFAR).}
Figure~\ref{fig:ba_topology} shows the Barabasi--Albert graph
($m{=}2$, $p{=}0.1$, 122 edges) used in the sparse heterogeneous
setting (Table~\ref{tab:agnostic_skewed}). The scale-free degree
distribution produces a small number of hubs and a long tail of
degree-2 leaves; hubs are assigned EfficientNet-B0 (stars) and the
remaining nodes MobileNetV2 (circles).

\begin{figure}[htbp]
\centering
\includegraphics[width=0.85\linewidth]{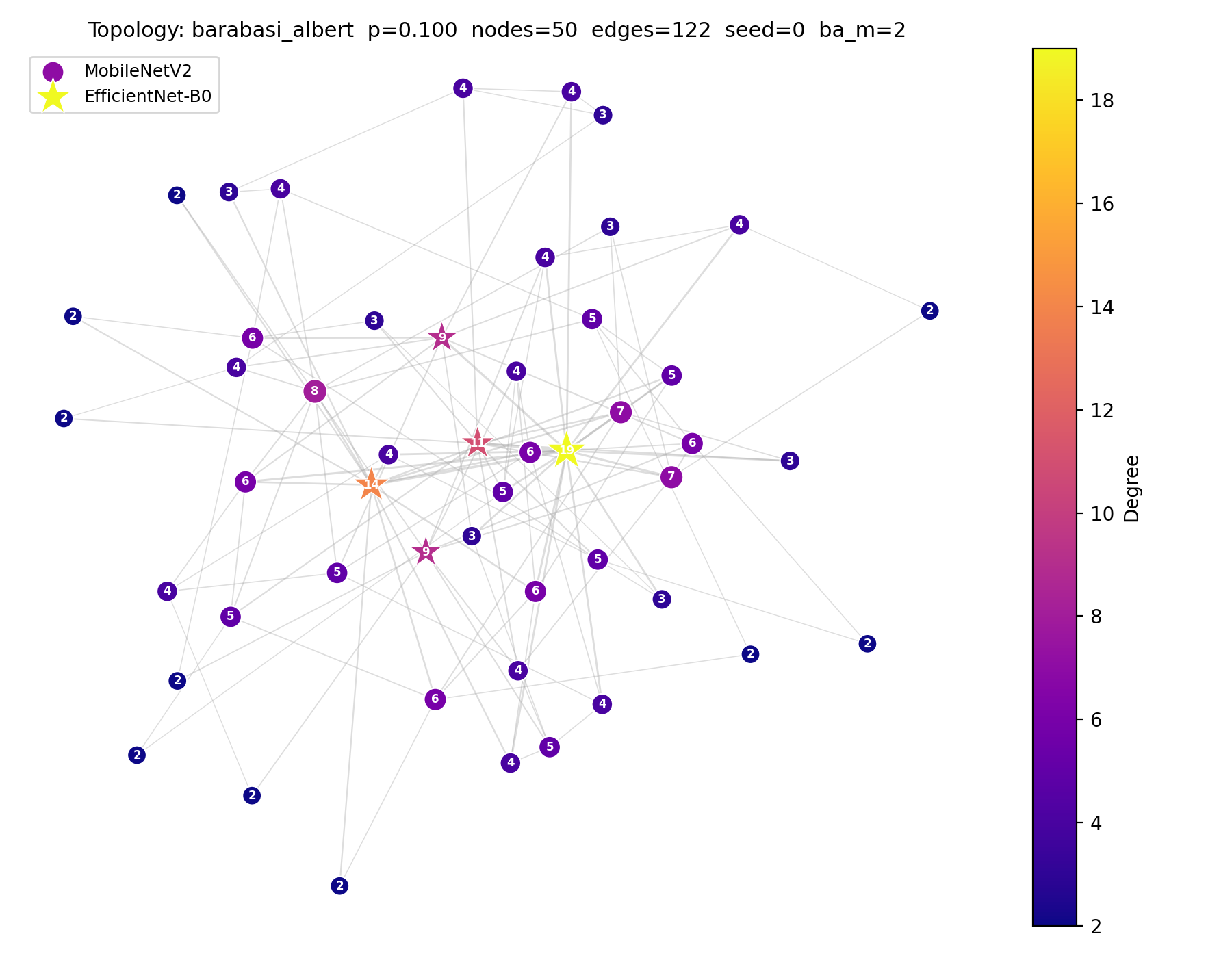}
\caption{Sparse heterogeneous topology (BA, $m{=}2$, $p{=}0.1$,
$n{=}50$, seed~0; 122 edges). Node color encodes degree;
EfficientNet-B0 hubs are stars, MobileNetV2 nodes are circles.}
\label{fig:ba_topology}
\end{figure}

\paragraph{Dense homogeneous topology (CIFAR).}
Figure~\ref{fig:uniform_topology} shows the uniform random graph
($p{=}0.5$, 612 edges) used in the dense homogeneous setting
(Table~\ref{tab:dense_homo}). All nodes run MobileNetV2. The
near-uniform degree distribution (min 18, max 30) and dense
connectivity are favorable to parameter-sharing baselines.

\begin{figure}[htbp]
\centering
\includegraphics[width=0.85\linewidth]{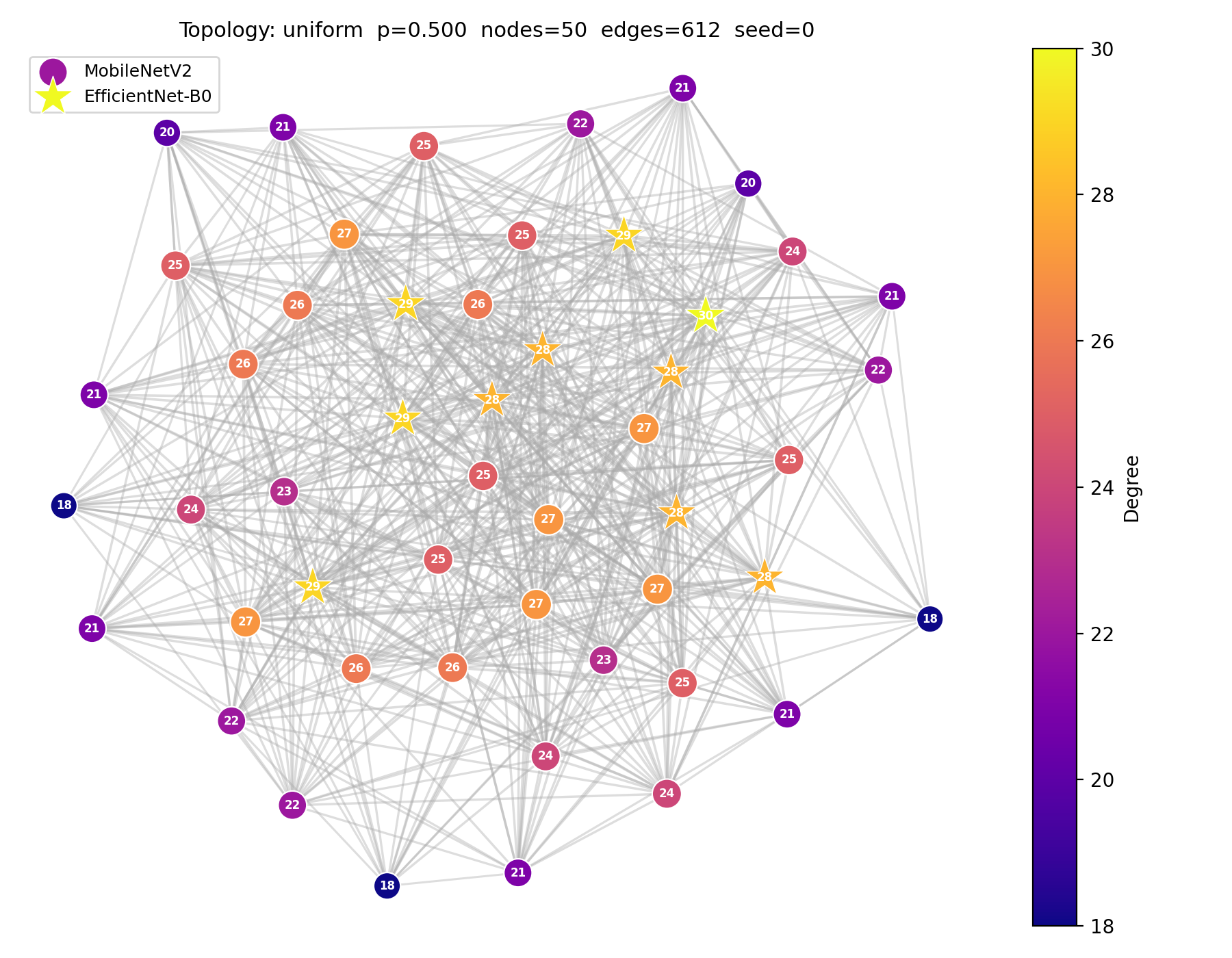}
\caption{Dense homogeneous topology (uniform random graph,
$p{=}0.5$, $n{=}50$, seed~0; 612 edges). All nodes run
MobileNetV2. The near-uniform degree distribution (min 18, max 30)
and dense connectivity favor parameter-sharing baselines.}
\label{fig:uniform_topology}
\end{figure}

\end{document}